\documentclass{article}
\usepackage{log_2026}						% for camera-ready version

\usepackage{booktabs}						% professional-quality tables
\usepackage{multirow}						% tabular cells spanning multiple rows
\usepackage{amsfonts}						% blackboard math symbols
\usepackage{graphicx}						% figures
\usepackage{duckuments}						% sample images
\usepackage{amsthm}

\usepackage{thmtools}
\usepackage{thm-restate}
\theoremstyle{remark}
\newtheorem{remark}{Remark}
\theoremstyle{lemma}

\usepackage[numbers,compress,sort]{natbib}	% for numerical citations
\title[A Label Noise Robust Algorithm for Graph
Data using Neighborhood Label Distribution]{A Label Noise Robust Algorithm for Graph Data using Neighborhood Label Distribution}

\author[Kumar and Hemachandra]{%
Pintu Kumar\\
Indian Institute of Technology Bombay \\
\email{pintuk@iitb.ac.in}\And
Nandyala Hemachandra\\
Indian Institute of Technology Bombay\\
\email{nh@iitb.ac.in}
}

\begin{document}

\maketitle

\begin{abstract}
Graph Convolutional Networks (GCNs) have demonstrated exceptional performance on graph-structured data. However, their success heavily relies on the availability of clean labels, an assumption that often fails in real-world scenarios. The learnability of GCNs is intrinsically linked to the neighborhood structure of nodes, making it critical to understand how this structure changes in the presence of label noise. In this work, we investigate how the neighborhood label distribution of a node changes under label noise. We mathematically derive that label noise smooths this distribution, which inherently increases its entropy and severely decreases interclass distinguishability. To address this challenge, we propose Leave-one-out Entropy Node Screening (LENS), an end-to-end framework that identifies noisy nodes by measuring the reduction in neighborhood entropy when a specific node is removed. LENS further incorporates a relabeling mechanism to recover and preserve valuable feature information instead of entirely discarding the identified noisy nodes. Crucially, the end-to-end computational overhead of LENS is remarkably low. Extensive evaluations demonstrate that our algorithm achieves state-of-the-art node classification accuracy while executing an order of magnitude faster than existing noise-robust approaches.
\end{abstract}

\section{Introduction}
Graph Neural Networks (GNNs), and particularly message-passing algorithms like Graph Convolutional Networks (GCNs), have proven to be exceptionally effective for node classification tasks on graphs \cite{luo2024classic,Hamilton2020GraphRL,kumar2026graph}. Similar to most traditional machine learning models, GCNs operate under the strong assumption that clean labels are available for training. However, this is rarely the case in practice. Real-world data is frequently corrupted by label noise due to factors such as human annotation errors, reliance on cheap crowdsourcing platforms, or inherent ambiguities in the data collection process \cite{wang2019symmetric,qian2023robust,tut_nh,DeGLIf,kumar2025edn,cgnnYuan2023LearningOG,manwani2013noise,ghosh2015making,pignnDu2021NoiserobustGL,crgnnLi2024ContrastiveLO,TSS,chen2024erase,Patrini2016MakingDN}. While researchers have dedicated significant effort to solving the label noise problem for independent and identically distributed (i.i.d.) data \cite{sastry2017robust,wang2019symmetric,Patrini2016MakingDN,Han2018CoteachingRT,Bahri2020DeepKF,tripathi2019cost,tut_nh}, these algorithms consistently underperform when applied to interconnected graph data \cite{dai2021nrgnn,qian2023robust,wang2024noisygl}. This led to a strand of research on learning in the presence of label noise, specifically for graph data \cite{dai2021nrgnn,qian2023robust,TSS,DeGLIf,kumar2025edn,yuan2023learninggraphslabelnoise,pignnDu2021NoiserobustGL,crgnnLi2024ContrastiveLO,RncglnZhu2024RobustNC,chen2024erase,wang2024noisygl}. A common assumption is that Traditional Noise-Robust Algorithms (TNRAs) fail on graph data simply because traditional algorithms ignore graph topology. Yet, empirical evidence suggests a deeper underlying issue: even when the standard neural network backbone of a TNRA is replaced with a topology-aware GCN, the algorithms still fail to yield the improvements seen in standard settings \cite{wang2024noisygl}. This indicates that graphs interact with label noise in a fundamentally different manner than their i.i.d. counterparts do.

To understand how graphs interact with label noise, it is crucial to examine how this noise impacts the fundamental structural quantities that dictate learning on graph data. A substantial body of work has established that a node's local neighborhood structure plays a vital role in how effectively a GCN can learn and aggregate information across a network \cite{mahomophily,Zhao2022NeighborhoodPI,Kohn2024EdgeSplittingMN}. In this work, we investigate how label noise distorts these foundational neighborhood patterns. Specifically, we analyze the mathematical change in the Neighborhood Label Distribution (NLD) of a node in the presence of noise. Although label noise robustness and the role of neighborhood on graph learning have both been studied extensively in isolation, bridging them provides critical insights. This work represents a first attempt at exploring the intersection of these two domains, providing a useful perspective on how noise structurally compromises graph learnability before any node filtering or robustness measures can be applied. Towards this, we ask: 

\begin{center}
    \emph{How does the injection of label noise perturb a node's Neighborhood Label Distribution? 
    \\ How can this understanding be used to identify noisy labeled nodes?
    \\ And, how does this perturbation impact class distinguishability?}
\end{center}

In this study, we theoretically demonstrate that label corruption acts as a smoothing agent on the NLD. This smoothing increases the entropy of NLD, making it progressively more Uniform distribution. Motivated by these theoretical insights, we introduce Leave-one-out Entropy Node Screening (LENS). LENS is an end-to-end, computationally efficient framework designed to identify noisy labeled nodes. LENS evaluates the change in neighborhood entropy when a specific node is computationally masked. Based on this, it classifies a label as clean or noisy. LENS achieves better performance than compared algorithms at a fraction of the computational cost. 

A direct consequence of this increased NLD entropy is a degradation in interclass distinguishability. As NLD becomes more uniform, the unique structural signatures of different classes begin to blend. We theoretically show this as a loss of interclass distinguishability. It is worth noting that NLD is node-dependent, whereas interclass distinguishability is a graph aggregate quantity. The main contributions of this work are as follows:

\textbf{1.} \textit{Theoretical Formulation of NLD Smoothing:} We model the transformation of a node's neighborhood label distribution (NLD) under various noise models, proving that corruption smooths the distribution and hence, increases NLD entropy.

\textbf{2.} \textit{A Novel Screening Framework (LENS):} We present an entropy-based, end-to-end screening mechanism that effectively isolates noisy nodes by analyzing leave-one-out entropy differentials. LENS outperforms competing algorithms while requiring significantly less computational power.

\textbf{3.} \textit{Analysis of Inter-class Distinguishability:} We establish that NLD smoothing directly erodes the structural boundaries between classes. We provide a formal proof for this structural collapse, under the Symmetric Label Noise (SLN) condition.

\section{Preliminaries}
Let $G = (V, E)$ be a graph where $V$ is the set of nodes and $E$ is the set of edges. Let $Y = \{1, 2, \dots, K\}$ denote a set of discrete classes, and let $y_v \in Y$ represent the ground-truth class label of a node $v \in V$. Let $N(v)$ denote neighborhood of node $v$, and $d_v = |N(v)|$, denote degree of node $v$. For node $v$, its noisy label is denoted by $\hat{y}_v$. 

\subsection{Noise Models}

\label{sec:noise_mod}
 We first describe various noise models that we consider. Our choice of noise models aligns with those commonly adopted by existing noise-robust algorithms \cite{dai2021nrgnn,qian2023robust,DeGLIf}
 
\subsubsection{Symmetric Label Noise (SLN)}
\label{sec:sln}
SLN \citep{tut_nh,tripathi2019cost,ghosh2015making,DeGLIf} refers to a type of label noise where the mislabelled samples are equally likely to be assigned any of the possible labels. Mathematically, $P(\hat{y}_v = l | y = m) = p/(K-1)=p_K$ where $l\neq m $ represent the possible labels and $p \in [0,1]$ is the noise level. The transition probability matrix for SLN ($Q_{SLN}$) is in Equation 
% \begin{equation}
% \label{matrix:sln}   
% Q_{sln} =
% \begin{bmatrix}
% 1- p &p/(K-1)&p/(K-1)& \ldots &p/(K-1)\\ 
% p/(K-1) &1-p &p/(K-1)& \ldots &p/(K-1)\\ 
% \vdots & \ddots & \ddots & \ddots & \vdots \\ 
% p/(K-1) & \ldots & \ddots &1-p &p/(K-1)\\ 
% p/(K-1) &p/(K-1)& \ldots &p/(K-1)&1-p
% \end{bmatrix}
% \end{equation}

\subsubsection{Pairwise Noise (PWN)}
\label{sec:pairwise_noise}
 The motivation behind Pairwise Noise (PWN) \citep{tut_nh,tripathi2019cost,ghosh2015making,DeGLIf} is that one is more likely to mislabel two similar classes. For Pairwise Noise, the true label of each node is flipped to a specific adjacent class with probability $p$. Unlike SLN, the noise is directional, where class $i$ is strictly corrupted into class $(i+1 \pmod K)$. The transition probability matrix is in Equation \ref{eq:pair_matrix} 

% \begin{equation}
% Q_{pw}=
% \begin{bmatrix}
% 1- p\ \ & p &\ \ 0  & \ldots & \ \ \ 0 \\ 
%  0& 1-p \ \  & \ \ \ \ p\ \ &  & \ \ \ 0\\ 
% \vdots &  & \ \ \ddots & \ \ \ddots &\ \ \  \vdots\\ 
%  0&  &  & \ \ \ 1-p & \ \ \ p\\ 
% p &0  & \dots &\ \ \ 0  &\ \ \ 1-p
% \end{bmatrix}
% \label{eq:pair_matrix}
% \end{equation}

\ref{matrix:sln}.
\noindent
\begin{minipage}{0.5\textwidth}
  \begin{equation}
    \label{matrix:sln}   
    \resizebox{0.85\textwidth}{!}{$
    Q_{sln} =
    \begin{bmatrix}
      1- p &p_K&p_K& \ldots &p_K\\ 
      p_K &1-p &p_K& \ldots &p_K\\ 
      \vdots & \ddots & \ddots & \ddots & \vdots \\ 
      p_K & \ldots & \ddots &1-p &p_K\\ 
      p_K &p_K& \ldots &p_K&1-p
    \end{bmatrix}
    $}
  \end{equation}
\end{minipage}%
\begin{minipage}{0.5\textwidth}
  \begin{equation}
    \label{eq:pair_matrix}
    \resizebox{0.85\textwidth}{!}{$
    Q_{pw}=
    \begin{bmatrix}
      1- p & p & 0  & \ldots & 0 \\ 
      0 & 1-p & p & \dots & 0\\ 
      \vdots &  & \ddots & \ddots & \vdots\\ 
      0 &  &  & 1-p & p\\ 
      p & 0  & \dots & 0  & 1-p
    \end{bmatrix}
    $}
  \end{equation}
\end{minipage}

\subsubsection{{Doubly stochastic noise models}}
In general, noise transition matrices are row-stochastic because the probabilities of a true label flipping to any of the available corrupted classes must sum to 1. A doubly stochastic matrix represents a special class of these noise models where the probabilities within each column also sum to 1. Both SLN and PWN are specific instances that are special cases of doubly stochastic matrices.

\subsection{Neighborhood Label Distribution}
\label{sec:nld_main}
For a specific class $j \in Y$, let $c_{v,j}$ denote the exact integer count of neighbors of $v$ that belong to class $j$, and $C_v:=[c_{v,1},\ldots,c_{v,K}]$. The neighborhood label distribution (NLD) is the fraction of neighboring nodes belonging to each class and is defined as 

$$l_v:= [l_{v,1},\ldots,l_{v,K}]=\left[\frac{c_{v,1}}{d_v},\ldots,\frac{c_{v,K}}{d_v}\right]$$

\section{Impact of Label Noise on NLD }
\label{sec:nld_res}
To quantify the impact of label corruption, we first analyze the transformation of a node's clean NLD in the presence of different label noise models. $\hat{l}_v$ denotes NLD of a noisy graph.

\subsection{Change in NLD under SLN}
In this section, we analyze the change in NLD in the presence of SLN (Section \ref{sec:sln}). 

\begin{restatable}{theorem}{nldsln}
\label{thm:sln_nld}
% Let $l_v = [l_{v,1}, l_{v, 2},\cdots,l_{v, K}]$   denote the initial Neighborhood Label Distribution (NLD) of a given node $v$ in a $K$-class setting. Under Symmetric Label Noise (SLN) with a noise rate $p$,  the changed NLD,  $\hat{l}_v = [\hat{l}_{v, 1}, \hat{l}_{v, 2}, \cdots, \hat{l}_{v, K}]$, is given by:
Let $l_v = [l_{v,1}, l_{v, 2},\cdots,l_{v, K}]$   denote the initial Neighborhood Label Distribution (NLD) of a given node $v$ in a $K$-class setting. Under Symmetric Label Noise (SLN) with a noise rate $p$,  the \textcolor{black}{expected probability mass of class $i$, denoted as $\mathbb{E}[\hat{l}_{v,i}]$}, is given by:
\begin{equation}
\label{eq:sln_main} 
\mathbb{E}[\hat{l}_{v, i}] = l_{v, i} \left( 1 - \frac{pK}{K-1} \right) + \frac{p}{K-1}, ~~~~ \forall~ i = 1, \cdots, K. 
\end{equation}
 \end{restatable}

% Let $l_v = [l_1, l_2,\cdots,l_K]$  \textcolor{blue}{as above, we need $l_{v}$ and $l_{v, 1}, \cdots, l_{v, K}$ \textcolor{black}{in line 101 I have mentioned that we are dropping v}} denote the initial Neighborhood Label Distribution (NLD) of a given node in a $K$-class setting. Under Symmetric Label Noise (SLN) with a noise rate $p$,  the changed NLD,  $\hat{l} = [\hat{l}_1, \hat{l}_2, \cdots, \hat{l}_K]$, is given by:
% \begin{equation}
% \label{eq:sln_main} 
% \hat{l}_{i} = l_{i} \left( 1 - \frac{pK}{K-1} \right) + \frac{p}{K-1}, ~~~~ \forall~ i = 1, \cdots, K. 
% \end{equation}
% \end{restatable}

% \textcolor{blue}{Let $l=(l_1,l_2,\dots,l_K)$ denote the initial Neighborhood Label Distribution (NLD) of a given node in a $K$-class setting. Under Symmetric Label Noise (SLN) with a noise rate $p$,  the changed/new NLD,  $\hat{l} = (\hat{l}_1, \hat{l}_2, \cdots, \hat{l}_K)$, is such that} 
% \begin{equation}
% \label{eq:sln_main} 
% \hat{l}_i = l_i \left( 1 - \frac{pK}{K-1} \right) + \frac{p}{K-1}, ~~~~ \forall~ i = 1, \cdots, K.  

% %\textcolor{blue}{this being a convex combination, is higher -- alternate proof.}
% \end{equation}

Proof for Theorem \ref{thm:sln_nld} is in Appendix \ref{proof:nldsln}.

\textbf{The Smoothing Factor}
 The difference between probability mass any two classes $i$ and $j$ after the noise is applied to the label of node $v$ is given by :$$|\mathbb{E}[\hat{l}_{v,i}] - \mathbb{E}[\hat{l}_{v,j}]| = |(l_{v,i} - l_{v,j)}| \left |1 - \frac{pK}{K-1} \right |$$
 
 As $2\ge \frac{pK}{K-1} >0$, so the term $\left( 1 - \frac{pK}{K-1} \right)$ acts as a contraction (For $1>p > 0$, this multiplier is strictly less than $1$). The gap of the majority class and the minority classes shrinks. The entire distribution is pulled relentlessly toward the uniform distribution $\left[ \frac{1}{K}, \frac{1}{K}, \dots, \frac{1}{K} \right]$. To formalize this intuition, we compare the entropy $H(\cdot)$ of $l_v$ and $\mathbb{E}[\hat{l}_v]$.

% \textcolor{blue}{we can have the variance expression as a lemma here}

% \textcolor{blue}{can that expression help us to have a better algo (more informative)?}

\subsubsection{Change in NLD entropy by SLN} Next, we investigate the effect of noise on the entropies of NLDs. 

%\textcolor{blue}{We quantify the effect of SLN in terms of entropy of noise label distributions.}
\begin{restatable}{corollary}{slnsln}
\label{cor:sln_entropy}
For SLN, we have $H(\mathbb{E}[\hat{l}_v]){\ge} H(l_v)$ for SLN
\end{restatable}

%{$H(\hat{l})\overset{?}{\ge} H(l)$ for SLN}

% \textcolor{blue}{this proof in Appendix or here? if here, we need `proof' environment? } 
\begin{proof}Let $\alpha:=1-\frac{pK}{K-1} \in (0, 1)$, then using Equation \ref{eq:sln_main}: $\mathbb{E}[\hat{l}_v]=\alpha l_v +(1-\alpha) U.$
Where $U$ is a uniform distribution over $K$ elements. As entropy is a concave function, and  uniform distribution has maximum entropy \cite{cover1999elements},

\begin{equation}
\label{eq:nldchangesln}
    H(\mathbb{E}[\hat{l}_v])\geq \alpha H(l_v)+(1-\alpha) H(U)\geq \alpha H(l_v)+(1-\alpha) H(l_v)=H(l_v)
\end{equation}

So, $\mathbb{E}[\hat{l}_v]$ is closer to uniform distribution compared to $l_v$ \cite{cover1999elements}.
\end{proof}

\subsection{Impact of PWN on NLD Entropy}
In this section, we analyze the change in NLD under PWN (Section \ref{sec:pairwise_noise}).

\begin{restatable}{theorem}{nldpwn}
\label{thm:pwn_nld}
% Let $l_v=[l_{v, 1}, l_{v,2},\cdots,l_{v, K}]$ denote the initial NLD of a given node $v$ in a $K$-class setting. Under Pairwise Noise (PWN) with a noise rate $p$, the changed NLD, $\hat{l}_v=[\hat{l}_{v, 1},\cdots, \hat{l}_{v, K}]$, is given by:
Let $l_v=[l_{v, 1}, l_{v,2},\cdots,l_{v, K}]$ denote the initial NLD of a given node $v$ in a $K$-class setting. Under Pairwise Noise (PWN) with a noise rate $p$, \textcolor{black}{expected probability mass of class $i$, denoted as $\mathbb{E}[\hat{l}_{v,i}]$}, is given by:
\begin{equation}
\label{eq:pwn_main} 
\mathbb{E}[\hat{l}_{v, i}] = l_{v, i}(1-p) + l_{v, i-1}p ~~~~ \forall~ i = 1, \cdots, K.
\end{equation}
where $l_{v, 0} \equiv l_{v, K}$. Furthermore, $H(\mathbb{E}[\hat{l}_v]) \ge H(l_v)$.
\end{restatable}

This theorem demonstrates that pairwise noise acts as a cyclic smoothing operation on the neighborhood distribution. The inevitable increase in entropy indicates that the NLD becomes progressively more uniform, which directly changes the local topology of the graph. Proof for Theorem \ref{thm:pwn_nld} is provided in Appendix \ref{proof:nldpwn}.

% \section{Change in NLD in presence of Pairwise Noise (PWN)}

% Using similar notation as in SLN. To find the new probability mass for our majority class $1$, we look at where its mass comes from. It retains its own mass with probability $1-p$, and it receives the corrupted mass from the "previous" class $K$ with probability $p$.$$\hat{l}_1 = l_1(1-p) + l_K p$$For any other arbitrary class $i$ (where $i \neq 1$), its new expected probability mass is:$$\hat{l}_i = l_i(1-p) + l_{i-1} p$$

% \textbf{Alternate proof:}
% Similar to SLN case, $c_{\mathrm{stay}}\sim Binomial(c_j,1-p)$ and $c_{\mathrm{influx}}\sim c_{j-1},p$ ($c_{0}\equiv c_K$). Then,

% $$E[\hat{l}_j]=l_j(1-p)+l_{j-1}p$$
% and
% $$Var(l_j)=\frac{p(1-p)}{d}(c_j+c_{j-1})$$
% \textcolor{black}{image validating above two formula}

% \subsection{$H(\hat{l})\geq H(l)$ for PWN}
% For $l=[l_1\ldots,l_K]$, define $l_{\mathrm{shift}}=[l_K,l_1,\ldots,l_{K-1}]$.
% Using Equation \ref{eq:nldpwchange}
% $$l_j=l_j(1-p)+l_{j-1}$$

% This is equivalent to 
% $$\hat{l}=(1-p)l+p l_{\mathrm{shift}}$$.
% See that $H(l)=H(l_{\mathrm{shift}})$. This and the concave property of the entropy function gives

% \begin{equation*}
%     \begin{aligned}
%         H(l')&\geq(1-p)H(l)+pH(l_{\mathrm{shift}})=H(x)
%     \end{aligned}
% \end{equation*}

\subsection{Change in NLD Entropy by Doubly Stochastic Noise Models}
% \textcolor{blue}{this doubly stochastic noise model includes SL:N and PWN as special cases.} 

% \textcolor{blue}{we now generalise the above results to a broader class of noise models}

% \textcolor{blue}{to update the contributions and abstract}

In this section, we investigate whether the entropy of the NLD always increases under any arbitrary noise model. We consider the broad class of doubly stochastic noise matrices, which includes SLN and PWN as special cases.

\begin{restatable}{theorem}{nlddb}
\label{thm:doubly_stochastic}
Let $l_v$ be the initial NLD of node $v$ and $Q$ be any doubly stochastic noise matrix. The \textcolor{black}{expected} NLD of node $v$ of the noisy labeled graph, denoted as $\textcolor{black}{\mathbb{E}[l_vQ]}$, satisfies $H(\mathbb{E}[\hat{l}_v]) \geq H(l_v)$. 
%\textcolor{green}{NLD of the node of the noisy labeled graph}
\end{restatable}

The proof relies on the Birkhoff-von Neumann theorem \cite{marshall1979inequalities} and is provided in Appendix \ref{app:proof_arbitrary}.

\begin{remark}
\label{rem:counterexamples}
The monotonic increase in NLD entropy guaranteed by Theorem \ref{thm:doubly_stochastic} relies strictly on the noise matrix $Q$ being doubly stochastic and uniformly applied across all nodes. If the label noise process deviates from these conditions, the entropy of the neighborhood distribution is not guaranteed to increase. Specifically, entropy can decrease if:
\begin{enumerate}
    \item The global noise transition matrix $Q$ is row-stochastic but not doubly stochastic.
    \item The noise is node-dependent or class-dependent (even if the individual node-specific transition matrices are doubly stochastic).
 \item Concrete counterexamples demonstrating both of these scenarios are detailed in Appendix \ref{app:counterexamples}.
\end{enumerate}
\end{remark}
{\color{black}
Next, we derive a high probability bound on the difference between the entropy of the expected noisy NLD and the realized noisy NLD. The proof in \ref{app:proof_hpb} uses Bretagnolle-Huber-Carol inequality \cite{vanderVaartWellner2003ed2}
 and Fannes-Audenaert continuity of entropy \cite{Audenaert2006ASC}. 
\begin{restatable}{theorem}{hpb}
\label{thm:hpb}
    Let $l_v$ denote the initial NLD of a given node $v$ with degree $d_v$ in a K-class setting. Let $\hat{l}_v$ denote the NLD of its noisy graph obtained by corrupting nodes independently according to a common
row-stochastic transition matrix Q. For a given $\delta\in(0,1)$, the following holds with probability at least $1-\delta$
    \begin{equation}
    |H(\hat{l}_v)-H(\mathbb{E}[\hat{l}_v])|\leq \sqrt{\frac{K\ln(2)+\ln(1/\delta)}{2~d_v}} \ln(K-1) + \ln 2
    \end{equation}
\end{restatable}}

\section{LENS: Leave-one-out Entropy Node Screening}
Based on the theoretical results obtained in Equation \ref{eq:nldchangesln}, alongside Theorems \ref{thm:pwn_nld} and \ref{thm:doubly_stochastic}, we establish that label noise systematically smooths a node's Neighborhood Label Distribution and increases its overall entropy. Motivated by this, we propose a novel, entropy-based methodology to effectively screen and isolate clean nodes from their corrupted counterparts.

\subsection{Identifying Noisy Nodes}
\label{sec:identify_noisy}
As established, label noise increases the entropy of a node's \textcolor{black}{expected} NLD. By extension, the presence of a noisy node $v$ in the neighborhood of a node $u$ is likely to increase the NLD entropy of $u$. Let $h(u):= H(\text{NLD}(u))$ denote the initial NLD entropy of $u$, and let $h_{-v}(u)$ denote the NLD entropy of $u$ when node $v$ is explicitly removed from its neighborhood. In a graph with noisy labels, if the removal of node $v$ reduces the NLD entropy of its neighbor $u$ (i.e., $h(u) - h_{-v}(u) > 0$), it is a strong indicator that $v$ possesses a noisy label (see the ``Identify noisy nodes'' block in Figure \ref{fig:algo_lens}). 

Because a single node $v$ simultaneously influences the neighborhoods of all its adjacent nodes, we must aggregate these localized entropy changes into a comprehensive scalar score to classify $v$ as either clean or noisy. To achieve this, we introduce the following two aggregation strategies, which define the two variants of the LENS framework:

    \textbf{LENS(Sum):} A node $v$ is classified as noisy if the cumulative entropy reduction across its neighborhood exceeds a threshold $\tau_s$, formulated as $\sum_{u \in N(v)} (h(u) - h_{-v}(u)) > \tau_s$. We set $\tau_S = 0$; while we empirically evaluated other values, a threshold of 0 consistently demonstrated the most robust performance across different graph topologies and noise levels. 

 \textbf{LENS(Vote):} A node $v$ is classified as noisy if the proportion of its neighbors that experience an entropy reduction exceeds a threshold $\tau_v$, formulated as $\frac{|\{u \in N(v) \mid h(u) - h_{-v}(u) > 0\}|}{|N(v)|} > \tau_v$. We employ a strict majority rule by setting $\tau_v = 0.5$.

The nodes identified as clean form the set $D_c$, while the identified noisy nodes form a set $D_n$. As illustrated in Figure \ref{fig:met_pw}, both aggregation methods effectively separate the clean and noisy nodes.

\begin{figure}[!h]
    \centering
    \includegraphics[width=1 \linewidth]{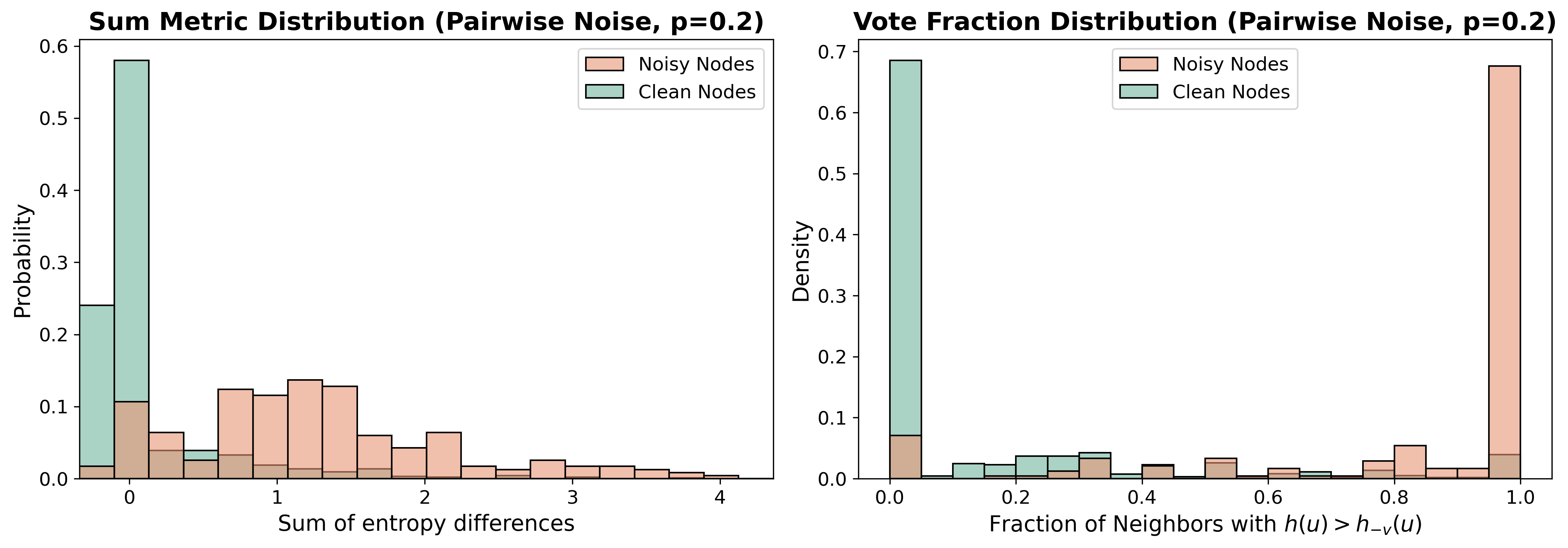}
    \caption{Distribution of the Sum and Vote aggregation metrics for clean and noisy nodes under 20\% Pairwise Noise. The clear separation between the distributions validates the effectiveness of both LENS variants in isolating corrupted labels.}
    \label{fig:met_pw}
\end{figure}
% \begin{figure}[!h]
%     \centering
%     \includegraphics[width=1 \linewidth]{figures/metric_distribution_Symmetric_0.2.png}
%     \caption{Metric distribution symmetric label noise, noise level 0.2}
%     \label{fig:met_sym}
% \end{figure}

% \begin{figure}
%     \centering
%     \includegraphics[width=1\linewidth]{figures/ete_lens.drawio.pdf}
%     \caption{end-to-end framework LENS \textcolor{blue}{the text in above blocks needs to be in darker shade}}
%     \label{fig:algo_lens}
% \end{figure}

\subsection{Change in Noise Level}
To capture the magnitude of noise reduction when moving from noisy data $D_{train}$ to $D_c$, we check the amount of noise left in $D_c$. The results for Cora dataset are reported in Table \ref{tab:denoising_results}. We see that both the methods (LENS(sum) and LENS(vote)) lead to a significant decrease in the noise.
\begin{table}[htpb]
\centering
\caption{{Noise levels (\%) in $D_c$ after applying Sum and Vote denoising methods under Symmetric Label Noise (SLN) and Pairwise Label Noise (PWN).}}
\label{tab:denoising_results}
\begin{tabular}{lccccccccc}
\toprule
\multirow{3}{*}{\textbf{Method}} & \multicolumn{9}{c}{\textbf{Initial Noise Level (\%)}} \\
\cmidrule(lr){2-10}
& \multicolumn{5}{c}{\textbf{SLN}} & \multicolumn{4}{c}{\textbf{PWN}} \\
\cmidrule(lr){2-6} \cmidrule(lr){7-10}
& \textbf{10} & \textbf{20} & \textbf{30} & \textbf{40} & \textbf{50} & \textbf{10} & \textbf{20} & \textbf{30} & \textbf{40}  \\
\midrule
LENS(Sum)  & 1.38 & 1.86 & 4.71 & 8.31  & 16.72 & 1.02 & 4.82 & 11.00 & 32.51  \\
LENS(Vote) & 1.70 & 3.58 & 6.87 & 11.24 & 21.29 & 1.97 & 7.13 & 14.31 & 34.34  \\
\bottomrule
\end{tabular}
\end{table}

% \begin{figure}
%     \centering
%     \includegraphics[width=1\linewidth]{figures/confusion_matrix_noise_0.4.pdf}
%     \caption{Confusion matrix for predicting noisy nodes for the Cora dataset in the presence of SLN with $p=0.4$. Left: Sum aggregation, Right: Vote aggregation \textcolor{black}{Rather than using figure replace by matrix, to save space, If even more space is required we can accomodate these values in running line}}
%     \label{fig:confusion_matrix}
% \end{figure}

\subsubsection{Addressing High Number of False Positives}
\label{sec:high_fp}
Next, we analyzed the confusion matrix for predicting a node as noisy. We observed a consistently high rate of False Positives across all noise levels for both the Sum and Vote aggregation methods. For example, for the Cora dataset at 40\% SLN noise, the Sum variant yields 954 true positives (TP), 1059 true negatives (TN), 599 false positives (FP), and 96 false negatives (FN). In comparison, for the Vote variant, TP is 867, TN is 1445, FP is 213, and FN is 183. We hypothesized that these misclassified nodes are the ones that have a heterogeneous neighborhood. 
% $$\text{Sum Aggregation} = \begin{bmatrix} 1059 & 599 \\ 96 & 954 \end{bmatrix}, \quad \text{Vote Aggregation} = \begin{bmatrix} 1445 & 213 \\ 183 & 867 \end{bmatrix}$$
To verify this, we analyzed the "clean local homophily" individually for the TP, TN, and FP nodes. Clean local homophily \cite{luan2024heterophilicgraphlearninghandbook} is the fraction of a node's neighbors that share its true, uncorrupted class label. The results for the Cora dataset SLN are shown in Figure \ref{fig:homo_diff}. The empirical results strongly support our hypothesis. Across all evaluated noise levels, the False Positives exhibit markedly lower local homophily than both True Negatives and True Positives. This demonstrates that LENS struggles specifically at class boundaries, where natural heterogeneity mimics the structural signature of label noise.

To prevent the aggressive filtering of valid boundary information, we introduce an additional verification step. In our refined approach, a node identified as noisy by the initial aggregation is considered truly noisy only if its NLD entropy is below a predefined threshold, \textcolor{black}{The entropy threshold was tuned on validation set; we used $t=1.2$ for Cora, Citeseer, Pubmed and Amazon Photo datasets, $t=1$, for DBLP dataset, and $t=0.9$ for Amazon ratings and Roman empire datasets.}.

\begin{figure}
    \centering
    \includegraphics[width=1\linewidth, height=4cm]{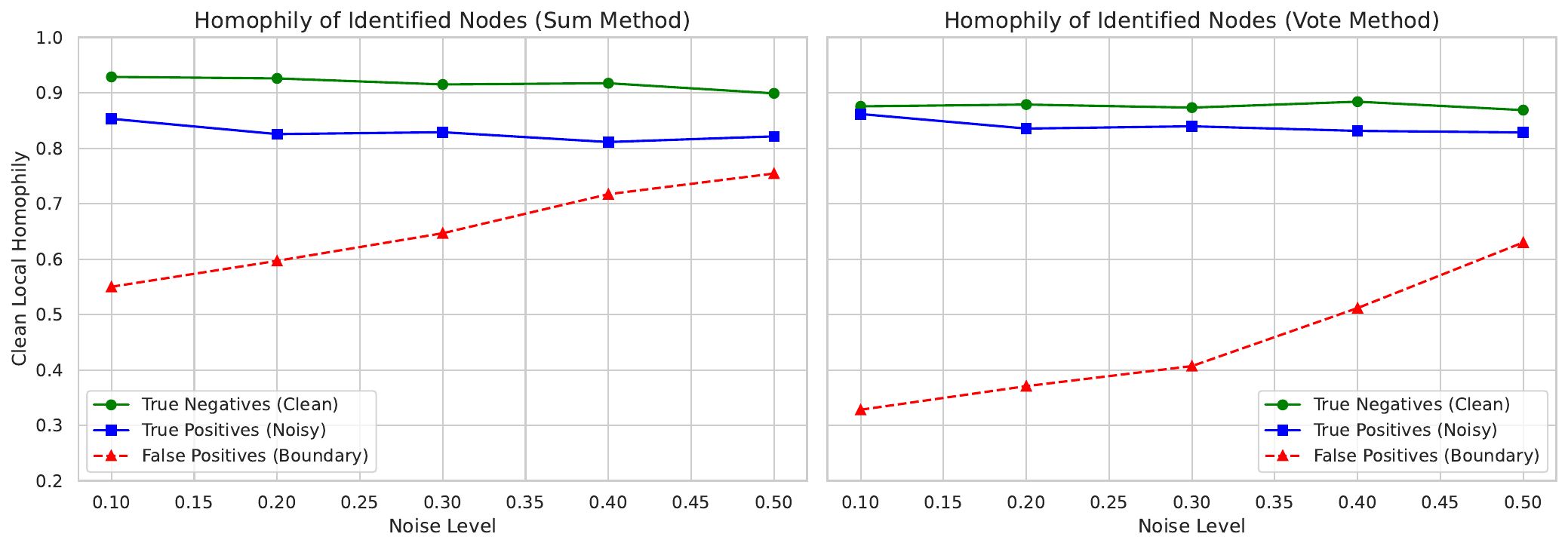}
    \caption{Clean local homophily of TP, TN, and FP nodes across varying noise levels on the Cora dataset. FPs consistently exhibit markedly lower homophily compared to correctly classified nodes.}
    \label{fig:homo_diff}
\end{figure}

\subsection{Pseudo Labeling}
We observed the effectiveness of LENS in distinguishing noisy from clean nodes. To calculate the NLD in Section \ref{sec:identify_noisy}, labels for all nodes are required, which are not available in a realistic scenario. To overcome this, we use a simple pseudo-labeling block to predict labels for the test and validation nodes (See Pseudo Label generation block in Figure \ref{fig:algo_lens}). We train a GCN on $D_{train}$ and use its prediction on test and validation as the label for the same. These labels are used only for calculating NLD in the next step. The \emph{final} end-to-end framework for LENS is in Figure \ref{fig:algo_lens}.

\begin{figure}[!h]
    \centering
    \includegraphics[width=1\linewidth]{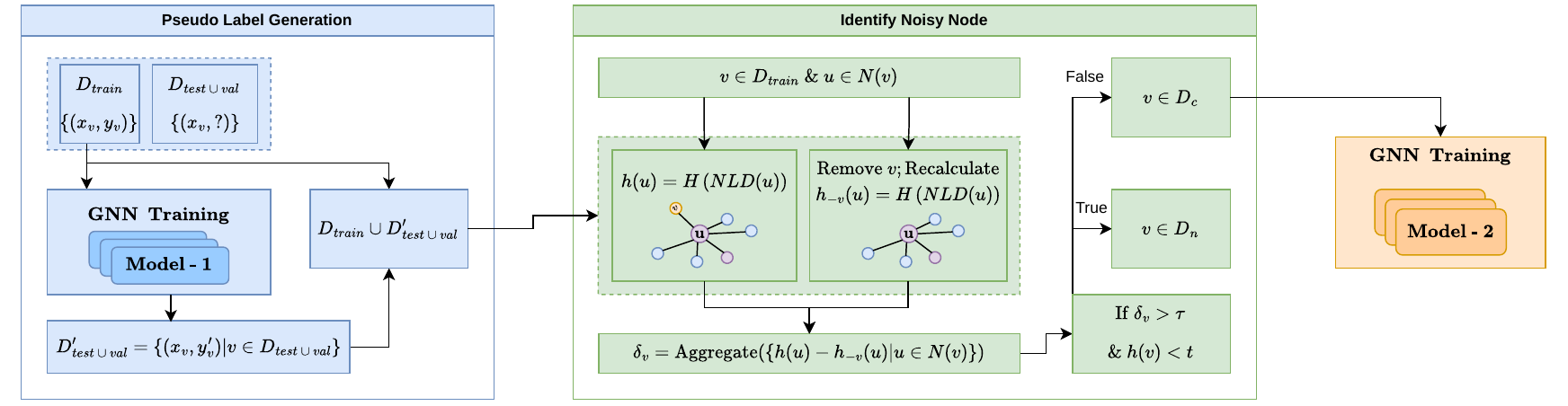}
    \caption{Workflow of the LENS framework. A pseudo-label generation block first generates labels for the test and validation sets. These pseudo-labels are subsequently used to compute the NLD entropy, enabling the identification of clean and noisy nodes by measuring the entropy reduction when a specific node is removed. A GCN is finally trained on clean nodes.}
    \label{fig:algo_lens}
\end{figure}

  % \textcolor{blue}{large white space between Fig 3 and Table 2}

\section{Evaluation of LENS for Node Classification on
Noisy Labeled Graph}
\subsection{Setup and Datasets}
\textbf{Noise Models and datasets used:} We use SLN (Section \ref{sec:sln}) and PWN (Section \ref{sec:pairwise_noise}) \citep{tut_nh,tripathi2019cost,ghosh2015making} models to inject noise in training data. We want to note that our algorithm does not use or estimate the noise model or noise level.We test our algorithm on benchmark datasets, namely, Cora\citep{Yang2016RevisitingSL}, Citeseer\citep{Yang2016RevisitingSL}, Amazon Photo\citep{shchur2018pitfalls}. All datasets are fetched from PyG library \cite{Fey/Lenssen/2019}.

\textbf{Implementation Details: }
In LENS, we employ two GNN models (see Fig. \ref{fig:algo_lens}). LENS can support any combination of GNN architectures. For fair comparison with baselines, we select both Model-1 and Model-2 as GCN \citep{kipf2016semi} with a single hidden layer having a dimension of 16. The output dimension matches the number of classes in the dataset. \textcolor{black}{Both models are trained for a maximum of 1000 epochs using early stopping and the Adam optimizer. L2 regularization is set to 5e-4 for Cora and Citeseer, and 1e-5 for Amazon Photo.} 
All experiments, performed on an NVIDIA RTX 4090 GPU (24GB VRAM) are repeated 5 times with random seeds and mean $\pm$ standard deviation are reported. 

\textbf{Baselines: } We compare our method with existing state-of-the-art noise-robust models for node classification: \textbf{GCN} \citep{kipf2016semi}, \textbf{Coteaching+}  \citep{Yu2019HowDD}, \textbf{NRGNN}  \citep{dai2021nrgnn}, \textbf{RTGNN} \citep{Qian2022RobustTO}, \textbf{CP} \citep{CPZhang2020AdversarialLA}, \textbf{CGNN} \citep{cgnnYuan2023LearningOG}, \textbf{CRGNN} \citep{crgnnLi2024ContrastiveLO}, \textbf{RNCGLN} \citep{RncglnZhu2024RobustNC}, \textbf{PIGNN} \citep{pignnDu2021NoiserobustGL}, \textbf{DGNN} \citep{NT2019LearningGN},\textbf{TSS} \citep{TSS}, \textbf{DeGLIF} \cite{DeGLIf}. For details about the methodology adapted by these models, see Appendix \ref{app:relatedworks}. For GCN, we use the implementation from the PyG  \citep{Fey/Lenssen/2019} library. Also, for a fair comparison, we use the implementation of NRGNN,  Coteach+ by \citet{dai2021nrgnn}, the implementation of CP, CGNN, CRGNN, RNCGLN, PIGNN, and DGNN by \citet{wang2024noisygl}, and the implementation of RTGNN \citep{Qian2022RobustTO}, DeGLIF \cite{DeGLIf} and TSS \citep{TSS}. 

\begin{table}[!htpb]
\centering
\caption{Node classification accuracy on the Cora and Amazon Photo datasets under Symmetric Label Noise (SLN) and Pairwise Label Noise (PWN).}
\label{tab:combined_accuracy_side_by_side}
\resizebox{\linewidth}{!}{
\begin{tabular}{ll ccccc cccc}
\toprule
\multirow{2}{*}{\textbf{Dataset}} & \multirow{2}{*}{\textbf{Algorithm}} & \multicolumn{5}{c}{\textbf{Symmetric Label Noise (SLN) (\%)}} & \multicolumn{4}{c}{\textbf{Pairwise Label Noise (PWN) (\%)}} \\
\cmidrule(lr){3-7} \cmidrule(lr){8-11}
& & \textbf{10} & \textbf{20} & \textbf{30} & \textbf{40} & \textbf{50} & \textbf{10} & \textbf{20} & \textbf{30} & \textbf{40}  \\
\midrule
\multirow{13}{*}{Cora} 
& GCN (CE) & $86.1 \pm 0.7$ & $85.2 \pm 1.1$ & $84.3 \pm 0.6$ & $81.3 \pm 1.3$ & $79.3 \pm 1.2$ & $85.1 \pm 0.4$ & $80.6 \pm 0.6$ & $73.8 \pm 1.9$ & $63.1 \pm 3.1$  \\
& Co-teaching + & $81.2 \pm 0.5$ & $79.6 \pm 0.8$ & $78.8 \pm 1.5$ & $76.4 \pm 1.8$ & $70.7 \pm 2.7$ & $78.6 \pm 1.3$ & $78.3 \pm 1.1$ & $74.9 \pm 3.1$ & $65.8 \pm 4.4$ \\
& NRGNN & $83.9 \pm 2.6$ & $82.0 \pm 2.5$ & $79.8 \pm 2.5$ & $76.5 \pm 3.7$ & $72.5 \pm 3.0$ & $84.6 \pm 0.7$ & $82.0 \pm 2.2$ & $78.3 \pm 2.6$ & $69.1 \pm 4.3$\\
& RTGNN & $79.0 \pm 2.7$ & $80.2 \pm 3.6$ & $83.0 \pm 0.4$ & $81.7 \pm 0.9$ & $77.6 \pm 0.7$ & $80.5 \pm 2.0$ & $82.0 \pm 2.4$ & $74.3 \pm 1.3$ & $59.6 \pm 1.5$ \\
& CGNN & $83.1 \pm 2.6$ & $82.2 \pm 3.9$ & $83.2 \pm 2.3$ & $78.6 \pm 4.2$ & $73.0 \pm 8.9$ & $83.6 \pm 2.0$ & $79.3 \pm 3.5$ & $74.5 \pm 5.9$ & $58.2 \pm 10.2$ \\
& RNCGLN & $79.8 \pm 0.8$ & $76.7 \pm 0.8$ & $70.4 \pm 3.4$ & $62.6 \pm 3.7$ & $51.7 \pm 2.9$ & $78.6 \pm 1.3$ & $71.1 \pm 2.0$ & $63.0 \pm 3.5$ & $51.5 \pm 5.1$ \\
& PIGNN & $83.3 \pm 1.9$ & $81.2 \pm 5.5$ & $78.8 \pm 4.2$ & $78.0 \pm 2.2$ & $75.9 \pm 8.1$ & $82.1 \pm 1.8$ & $78.8 \pm 2.7$ & $72.8 \pm 3.6$ & $54.2 \pm 9.7$ \\
& DGNN & $79.0 \pm 3.2$ & $67.9 \pm 11.9$ & $73.2 \pm 5.2$ & $61.9 \pm 16.5$ & $60.0 \pm 7.6$ & $79.4 \pm 2.8$ & $74.3 \pm 3.2$ & $65.5 \pm 6.3$ & $56.0 \pm 5.5$ \\
& DeGLIF & $86.4 \pm 0.4$ & $85.7 \pm 0.5$ & $84.3 \pm 0.9$ & $82.2 \pm 0.8$ & $79.6 \pm 1.7$ & $\mathbf{86.4 \pm 0.3}$ & $\mathbf{84.2 \pm 1.1}$ & $78.1 \pm 0.8$ & $\underline{71.4 \pm 2.8}$\\
& TSS & $86.0 \pm 0.3$ & $85.2 \pm 0.6$ & $83.5 \pm 0.7$ & $81.4 \pm 0.9$ & $80.2 \pm 0.7$ & $85.6 \pm 0.6$ & $83.0 \pm 0.8$ & $78.8 \pm 0.7$ & $68.9 \pm 1.3$ \\
& CP & $81.4 \pm 2.1$ & $79.0 \pm 3.0$ & $75.8 \pm 6.4$ & $70.1 \pm 9.9$ & $65.6 \pm 6.3$ & $81.9 \pm 3.1$ & $74.8 \pm 5.8$ & $66.4 \pm 2.3$ & $56.1 \pm 2.3$\\
& LENS (Sum)    & $\underline{86.6 \pm 0.3}$ & $\mathbf{86.0 \pm 0.6}$ & $\underline{85.3 \pm 0.3}$ & $\mathbf{82.7 \pm 1.1}$ & $\mathbf{81.6 \pm 1.1}$ & $\underline{86.3 \pm 0.7}$ & $\mathbf{84.2 \pm 0.9}$ & $\mathbf{79.2 \pm 1.9}$ & $70.8 \pm 2.7$  \\
& LENS (Vote)   & $\mathbf{86.8 \pm 0.3}$ & $\underline{85.9 \pm 0.6}$ & $\mathbf{85.8 \pm 0.5}$ & $\underline{82.6 \pm 1.2}$ & $\underline{81.3 \pm 0.7}$ & $\underline{86.3 \pm 0.3}$ & $\underline{83.6 \pm 0.9}$ & $\underline{78.9 \pm 1.8}$ & $\mathbf{71.9 \pm 2.3}$ \\
\midrule
\multirow{13}{*}{Citeseer} 
& GCN (CE) & $\underline{77.54 \pm 0.4}$ & $76.28 \pm 0.7$ & $75.1 \pm 0.5$ & $72.4 \pm 0.6$ & $70.66 \pm 1.5$ & $\mathbf{77.7 \pm 0.26}$ & $75.4 \pm 0.71$ & $70.0 \pm 0.77$ & $60.5 \pm 2.09$ \\
& Co-teaching + & $\mathbf{77.8 \pm 5.8}$ & $73.6 \pm 5.5$ & $71.8 \pm 7.2$ & $65.6 \pm 9.2$ & $61.7 \pm 8.3$ & $75.7 \pm 1.0$ & $71.0 \pm 2.1$ & $65.1 \pm 5.7$ & $57.8 \pm 12.1$ \\
& NRGNN & $74.3 \pm 1.0$ & $72.2 \pm 1.2$ & $71.4 \pm 1.7$ & $69.5 \pm 1.2$ & $67.9 \pm 1.5$ & $76.1 \pm 1.7$ & $73.0 \pm 2.1$ & $68.6 \pm 3.4$ & $57.6 \pm 5.1$ \\
& RTGNN & $76.5 \pm 0.5$ & $\underline{76.4 \pm 1.2}$ & $\mathbf{76.3 \pm 0.7}$ & $73.2 \pm 0.9$ & $\underline{72.7 \pm 0.6}$ & $77.3 \pm 0.3$ & $75.8 \pm 1.2$ & $71.3 \pm 1.8$ & $\underline{67.9 \pm 1.0}$ \\
& CGNN & $73.9 \pm 4.6$ & $71.4 \pm 7.2$ & $71.6 \pm 4.7$ & $68.8 \pm 8.1$ & $64.3 \pm 2.8$ & $77.0 \pm 1.3$ & $69.3 \pm 3.7$ & $65.8 \pm 4.7$ & $55.1 \pm 6.1$ \\
& RNCGLN & $69.4 \pm 2.6$ & $66.1 \pm 3.8$ & $63.8 \pm 4.7$ & $57.8 \pm 4.9$ & $47.4 \pm 2.6$ & $68.3 \pm 3.2$ & $62.4 \pm 3.8$ & $55.8 \pm 3.2$ & $47.6 \pm 3.7$ \\
& PIGNN & $73.2 \pm 5.3$ & $70.8 \pm 4.8$ & $71.0 \pm 4.7$ & $67.4 \pm 4.1$ & $60.5 \pm 7.4$ & $72.9 \pm 2.7$ & $68.8 \pm 4.4$ & $62.7 \pm 5.2$ & $51.3 \pm 6.6$ \\
& DGNN & $62.1 \pm 2.9$ & $56.1 \pm 3.8$ & $46.9 \pm 8.9$ & $38.7 \pm 8.9$ & $33.3 \pm 8.0$ & $60.9 \pm 4.7$ & $52.8 \pm 4.4$ & $49.3 \pm 8.3$ & $36.8 \pm 7.3$ \\
& DeGLIF & $77.0 \pm 1.3$ & $\mathbf{76.7 \pm 0.9}$ & $\underline{75.8 \pm 0.8}$ & $\mathbf{74.5 \pm 0.5}$ & $\mathbf{72.8 \pm 0.4}$ & $\underline{77.5 \pm 0.8}$ & $\underline{76.4 \pm 1.3}$ & $73.6 \pm 1.1$ & $\mathbf{68.9 \pm 1.8}$ \\
& TSS & $77.1 \pm 0.5$ & $\mathbf{76.7 \pm 0.4}$ & $75.2 \pm 0.9$ & $\underline{73.9 \pm 1.2}$ & $71.1 \pm 1.0$ & $76.8 \pm 0.5$ & $\mathbf{76.7 \pm 0.6}$ & $74.0 \pm 1.0$ & $64.7 \pm 1.0$ \\
& CP & $72.9 \pm 4.6$ & $74.5 \pm 3.5$ & $69.9 \pm 7.7$ & $69.1 \pm 12.3$ & $63.4 \pm 14.1$ & $73.9 \pm 3.9$ & $66.2 \pm 4.5$ & $59.4 \pm 6.5$ & $49.4 \pm 6.9$ \\
& LENS (Sum) & $76.9 \pm 1.1$ & $75.7 \pm 0.7$ & $74.7 \pm 1.1$ & $73.7 \pm 0.6$ & $71.6 \pm 1.2$ & $76.8 \pm 1$ & $75.8 \pm 0.7$ & $\mathbf{74.3 \pm 1.1}$ & $65.7 \pm 2.5$ \\
& LENS (Vote) & $77.3 \pm 1$ & $\underline{76.4 \pm 0.7}$ & $75.2 \pm 0.8$ & $\mathbf{74.5 \pm 0.9}$ & $72.0 \pm 1.2$ & $77.3 \pm 0.6$ & $76.0 \pm 0.9$ & $\underline{74.1 \pm 1.1}$ & $65.7 \pm 2.2$ \\
\midrule
\multirow{13}{*}{Amazon Photo} 
& GCN (CE)& $\underline{91.1 \pm 0.9}$ & $90.1 \pm 0.7$ & $87.6 \pm 0.7$ & $85.7 \pm 2.0$ & $81.9 \pm 2.0$ & $90.8 \pm 0.9$ & $85.2 \pm 1.3$ & $79.0 \pm 1.7$ & $67.5 \pm 3.6$  \\
& Co-teaching + & $82.2 \pm 1.6$ & $80.4 \pm 1.3$ & $73.7 \pm 1.8$ & $59.2 \pm 2.4$ & $47.9 \pm 2.0$ & $88.6 \pm 3.3$ & $84.1 \pm 5.8$ & $78.9 \pm 2.7$ & $72.9 \pm 3.1$ \\
& NRGNN & $68.1 \pm 6.8$ & $66.5 \pm 5.1$ & $60.0 \pm 7.1$ & $60.0 \pm 5.1$ & $51.5 \pm 5.9$ & $65.6 \pm 2.7$ & $67.3 \pm 5.1$ & $65.7 \pm 6.0$ & $52.6 \pm 12.2$\\
& RTGNN &  $82.2 \pm 5.3$ & \textcolor{black}{$84.0 \pm 2.2$} & \textcolor{black}{$78.8 \pm 6.4$} & \textcolor{black}{$83.5 \pm 4.7$} & \textcolor{black}{$79.9 \pm 5.0$}& $86.6 \pm 0.7$ & $87.1 \pm 1.3$ & $77.6 \pm 1.1$ & $75.2 \pm 6.4$\\
& CGNN & $66.4 \pm 17.6$ & $69.6 \pm 22.8$ & $54.2 \pm 26.4$ & $43.0 \pm 26.4$ & $40.3 \pm 26.7$ & $61.5 \pm 20.2$ & $58.2 \pm 20.1$ & $49.3 \pm 25.6$ & $41.4 \pm 23.3$ \\
& RNCGLN & $84.7 \pm 1.7$ & $82.7 \pm 3.3$ & $74.1 \pm 8.3$ & $64.2 \pm 5.6$ & $46.6 \pm 4.5$ & $84.6 \pm 2.8$ & $80.7 \pm 3.1$ & $71.2 \pm 4.5$ & $55.7 \pm 11.2$ \\
& PIGNN & $90.1 \pm 1.7$ & $88.1 \pm 1.8$ & $87.6 \pm 0.7$ & $82.5 \pm 5.2$ & $76.9 \pm 4.3$ & $89.8 \pm 0.8$ & $86.4 \pm 1.6$ & $80.1 \pm 5.1$ & $65.7 \pm 8.9$ \\
& DGNN & $65.3 \pm 23.5$ & $54.4 \pm 17.5$ & $57.2 \pm 15.4$ & $48.1 \pm 14.9$ & $44.1 \pm 17.1$ & $63.7 \pm 23.5$ & $60.0 \pm 16.6$ & $50.4 \pm 19.5$ & $48.5 \pm 9.5$ \\
& DeGLIF & $\mathbf{91.6 \pm 0.2}$ & $\underline{90.3 \pm 1.0}$ & $89.8 \pm 2.0$ & $87.1 \pm 1.6$ & $82.6 \pm 4.2$ & $89.2 \pm 1.3$ & $87.4 \pm 0.6$ & $83.2 \pm 2.1$ & $74.5 \pm 7.6$\\
& TSS           & $85.3 \pm 3.5$ & $86.8 \pm 2.1$ & $86.3 \pm 1.3$ & $83.6 \pm 2.0$ & $78.9 \pm 4.8$ & $83.6 \pm 4.1$ & $85.9 \pm 3.0$ & $80.6 \pm 4.9$ & $67.1 \pm 9.7$  \\
& CP & $90.6 \pm 0.7$ & $89.9 \pm 0.5$ & $\underline{89.9 \pm 1.1}$ & $\underline{87.3 \pm 2.8}$ & $83.0 \pm 0.9$ & $\mathbf{91.3 \pm 0.6}$ & $88.8 \pm 2.1$ & $80.8 \pm 4.5$ & $68.9 \pm 5.9$ \\
& LENS (Sum)    & $\underline{91.1 \pm 0.7}$ & $\underline{90.3 \pm 1.2}$ & $\mathbf{90.4 \pm 1.1}$ & $\mathbf{89.0 \pm 1.0}$ & $\mathbf{87.4 \pm 1.9}$ & $\underline{91.1 \pm 0.8}$ & $\underline{89.2 \pm 1.1}$ & $\underline{86.3 \pm 3.2}$ & $\mathbf{77.8 \pm 4.4}$  \\
& LENS (Vote)   & $\mathbf{91.6 \pm 0.4}$ & $\mathbf{91.1 \pm 0.7}$ & $\mathbf{90.4 \pm 0.7}$ & $\mathbf{89.0 \pm 1.0}$ & $\underline{86.7 \pm 1.7}$ & $91.0 \pm 1.3$ & $\mathbf{89.8 \pm 1.2}$ & $\mathbf{86.5 \pm 2.7}$ & $\underline{77.7 \pm 3.9}$  \\
\bottomrule
\end{tabular}
}
\end{table}
\subsection{Performance of LENS on Node Classification in Presence of Label Noise}

We compare LENS with SoTA algorithms for noise-robust learning for graph data. Results are presented in Table \ref{tab:combined_accuracy_side_by_side}. We observe that both LENS variants outperform SOTA noise-robust algorithms. To evaluate performance across all 27 benchmark instances (3 datasets $\times$ 9 noise conditions), we conducted the average ranked test and the Wilcoxon signed-rank test. Proposed LENS(vote) achieved the best overall rank of 2.13, followed by LENS(sum) (2.72), DeGLIF (2.93), and TSS (4.96). Pairwise Wilcoxon signed-rank tests confirm that LENS(vote) achieves statistically significant performance gains over key baselines: {1. vs. DeGLIF}: $p = 0.0206$, achieving a $+0.67\%$ mean accuracy gain. \textbf{2. vs. TSS}: $p = 0.000037$, achieving a $+2.54\%$ mean accuracy gain.

We also observe that the computation time for LENS is much lower than that of other noise-robust algorithms, we detail this in next section.

\subsection{Computational Time Comparison }
To evaluate the practical scalability of the proposed method, we present a comparative runtime analysis against established noise-robust GNN baselines. Table \ref{tab:execution_time_cora} summarizes the execution time required by each algorithm on the Cora dataset. \textcolor{black}{The recorded execution times represent the complete end-to-end pipeline for every algorithm, from receiving the injected noisy graph data to the final model prediction.} Times are averaged over 5 independent runs to minimize variance.

\begin{table}[htpb]
\centering
\caption{Execution time comparison of noise-robust algorithms for Cora dataset on 24GB RTX 4090}
\label{tab:execution_time_cora}
\resizebox{\linewidth}{!}{
\begin{tabular}{lc @{\hspace{2em}} lc @{\hspace{2em}} lc  @{\hspace{2em}} lc}
\toprule
\textbf{Algorithm} & \textbf{Time (s)} & \textbf{Algorithm} & \textbf{Time (s)} & \textbf{Algorithm} & \textbf{Time (s)} & 
\textbf{Algorithm} & \textbf{Time (s)} \\
\midrule
NRGNN  & 15.75 & CRGNN & 5.40  & RTGNN  & 25.24 & RNCGLN & 14.69 \\ CP    & 12.57 & PIGNN  & 7.53  &
CGNN   & 11.50 & DGNN  & 4.53  \\ DeGLIF & 68.03 &
TSS    &     3.21  & \textbf{LENS}  & 0.25  &        &    \\  
\bottomrule
\end{tabular}}
\end{table}

LENS achieves a total execution time of just 0.25 seconds, demonstrating superior computational efficiency over all comparative methods. Even when evaluated against the fastest baselines in the benchmark, such as TSS (3.21 s), DGNN (4.53 s), and CRGNN (5.40 s), LENS provides an order-of-magnitude speedup, executing more than 12 times faster than the second-fastest approach, and significantly faster than DeGLIF (second-best performing algorithm as per ranked test). These results confirm that LENS achieves robust performance while maintaining minimal computational overhead, making it exceptionally well-suited for resource-constrained or real-time graph learning pipelines.

\subsection{Relabeling Function}

While training a GCN exclusively on the clean node set $D_c$ demonstrates noise robustness, entirely discarding the identified noisy nodes may be suboptimal. Simply throwing away these noisy nodes may lead to a significant loss of valuable feature information within the graph. Instead, we attempt to relabel them to recover this data and improve overall model performance.

The relabeling mechanism is detailed in the Figure \ref{fig:algo_lens}. For each identified noisy node $v \in D_n$, we generate a label prediction $p_v$ using a GCN trained on $D_c$ (Model-2). To determine the appropriate new label, we evaluate the condition $\text{argmax}(p_v) \neq y_v$, where $y_v$ represents the original noisy label. Depending on the outcome of this check, the new label $y_v^*$ is assigned as either the primary prediction $\text{argmax}(p_v)$ or the second-highest prediction $2^{nd} \text{argmax}(p_v)$. This process generates a corrected subset of nodes, denoted as $D_n^*$. Finally, the original clean nodes and the newly relabeled nodes are merged to form the augmented dataset $D^* = D_c \cup D_n^*$. A comparison of LENS with and without relabeling for the Cora dataset is presented in the Table \ref{tab:relabeling_comparison}. We see that relabelling offers comparable results at low noise levels, but at high noise levels, it significantly improves the accuracy.

\begin{table}[htpb]
\centering
\caption{Comparison of LENS aggregation methods (Sum $\&$ Vote) with and without relabeling on the Cora dataset under Symmetric Label Noise (SLN) and Pairwise Label Noise (PWN).}
\label{tab:relabeling_comparison}
\resizebox{\linewidth}{!}{
\begin{tabular}{lllccccc}
\toprule
\multirow{2}{*}{\textbf{Noise Model}} & \multirow{2}{*}{\textbf{Method}} & \multirow{2}{*}{\textbf{Setup}} & \multicolumn{5}{c}{\textbf{Initial Noise Level (\%)}} \\
\cmidrule(lr){4-8}
& & & \textbf{10} & \textbf{20} & \textbf{30} & \textbf{40} & \textbf{50} \\
\midrule
\multirow{4}{*}{SLN} 
& \multirow{2}{*}{LENS (Sum)} 
& w/o Relabeling  & $86.64 \pm 0.31$ & $85.96 \pm 0.60$ & $85.30 \pm 0.29$ & $82.68 \pm 1.10$ & $81.62 \pm 1.05$ \\
& & with Relabeling & $86.48 \pm 0.63$ & $85.62 \pm 0.84$ & $85.88 \pm 0.52$ & $83.96 \pm 0.85$ & $82.64 \pm 0.45$ \\
\cmidrule{2-8}
& \multirow{2}{*}{LENS (Vote)} 
& w/o Relabeling  & $86.82 \pm 0.34$ & $85.92 \pm 0.60$ & $85.76 \pm 0.53$ & $82.64 \pm 1.16$ & $81.26 \pm 0.72$ \\
& & with Relabeling & $86.62 \pm 0.57$ & $85.94 \pm 0.42$ & $85.22 \pm 0.48$ & $83.50 \pm 0.88$ & $82.20 \pm 0.54$ \\
\midrule
\multirow{4}{*}{PWN} 
& \multirow{2}{*}{LENS (Sum)} 
& w/o Relabeling  & $86.30 \pm 0.67$ & $84.22 \pm 0.94$ & $79.18 \pm 1.91$ & $70.78 \pm 2.66$ & - \\
& & with Relabeling & $85.76 \pm 0.48$ & $84.06 \pm 0.64$ & $81.18 \pm 0.83$ & $74.56 \pm 3.83$ & - \\
\cmidrule{2-8}
& \multirow{2}{*}{LENS (Vote)} 
& w/o Relabeling  & $86.26 \pm 0.31$ & $83.56 \pm 0.89$ & $78.86 \pm 1.78$ & $71.86 \pm 2.32$ & - \\
& & with Relabeling & $86.02 \pm 0.31$ & $84.36 \pm 0.48$ & $80.74 \pm 1.46$ & $72.84 \pm 5.06$ & - \\
\bottomrule
\end{tabular}
}
\end{table}

% \textbf{Remark:} During our research, we also explored a different relabeling criterion. We attempted to assign a label to a node that maximally reduces the entropy of its neighborhood. However, we observed that the improvement yielded by that approach was limited.

\subsection{Sensitivity Analysis of Thresholds $\tau_s$ and $\tau_v$}
We tested various values for $\tau_s \in (-0.1, 1)$ and $\tau_v \in (0.3, 0.8)$. We observed that no values consistently outperformed $\tau_s=0$ and $\tau_v=0.5$; we chose to use these settings. This behavior is explained by Figure \ref{fig:met_pw}, particularly for the Vote variant, as the proportion of neighbors experiencing entropy reduction is heavily concentrated around 0 or 1. Also, in the Sum variant, the sum of entropy differences for most clean nodes is less than 0.

\subsection{Experiments on Additional Datasets}
{\color{black}
We also evaluate LENS on PubMed \citep{Yang2016RevisitingSL}, the large DBLP graph \cite{Bojchevski2017DeepGE}, and two heterophilic datasets: Amazon Ratings \cite{Platonov2023ACL} and Roman Empire \cite{Platonov2023ACL}. We apply the data splits from \cite{wang2024noisygl} for DBLP and Amazon Ratings, and the splits from \cite{DeGLIf} for PubMed. To ensure a rigorous comparison on heterophilic graph, we adopt the experimental setup from \cite{DeGLIf} for the Roman Empire dataset, which utilizes the DirGNNConv architecture \cite{Rossi2023EdgeDI}. We compare LENS against the respective baseline architecture (GCN or DirGNNConv) alongside the two top-performing noise-robust algorithms for each dataset, as identified in \cite{wang2024noisygl} and \cite{DeGLIf}.

\begin{table}[!htpb]
\centering
\caption{Node classification accuracy on the Pubmed, DBLP, Amazon Ratings, and Roman Empire datasets under Symmetric Label Noise (SLN) and Pairwise Label Noise (PWN).}
\label{tab:pubmed_accuracy_side_by_side}
\resizebox{\linewidth}{!}{
\begin{tabular}{ll ccccc cccc}
\toprule
\multirow{2}{*}{\textbf{Dataset}} & \multirow{2}{*}{\textbf{Algorithm}} & \multicolumn{5}{c}{\textbf{Symmetric Label Noise (SLN) (\%)}} & \multicolumn{4}{c}{\textbf{Pairwise Label Noise (PWN) (\%)}} \\
\cmidrule(lr){3-7} \cmidrule(lr){8-11}
& & \textbf{10} & \textbf{20} & \textbf{30} & \textbf{40} & \textbf{50} & \textbf{10} & \textbf{20} & \textbf{30} & \textbf{40} \\
\midrule
\multirow{5}{*}{Pubmed} 
& GCN & $85.0 \pm 0.3$ & $84.9 \pm 0.3$ & $84.5 \pm 0.5$ & $83.9 \pm 0.1$ & $81.8 \pm 0.4$ & $85.6 \pm 0.1$ & $86.1 \pm 0.0$ & $83.9 \pm 0.4$ & $72.9 \pm 1.3$ \\
& CRGNN & $87.3 \pm 0.5$ & $87.2 \pm 0.6$ & $85.9 \pm 0.4$ & $85.0 \pm 0.9$ & $45.5 \pm 3.4$ & $87.1 \pm 0.6$ & $85.9 \pm 0.4$ & $83.1 \pm 0.9$ & $76.6 \pm 2.1$ \\
& DeGLIF & $86.6 \pm 0.4$ & $86.1 \pm 0.5$ & $85.6 \pm 0.5$ & $85.2 \pm 0.9$ & $84.5 \pm 0.9$ & $87.3 \pm 0.4$ & $86.1 \pm 0.4$ & $84.9 \pm 0.6$ & $81.2 \pm 1.3$ \\
& LENS (Sum) & $86.5 \pm 0.3$ & $85.8 \pm 0.3$ & $85.3 \pm 0.8$ & $84.6 \pm 0.9$ & $83.3 \pm 0.6$ & $86.6 \pm 0.4$ & $86.5 \pm 0.4$ & $85.5 \pm 0.9$ & $81.8 \pm 1.4$ \\
& LENS (Vote) & $87.2 \pm 0.4$ & $86.7 \pm 0.3$ & $86.0 \pm 0.7$ & $85.2 \pm 0.9$ & $84.2 \pm 0.9$ & $87.5 \pm 0.3$ & $87.0 \pm 0.4$ & $86.1 \pm 0.5$ & $80.3 \pm 2.8$ \\
\midrule
\multirow{5}{*}{DBLP} 
& GCN & $75.2 \pm 1.0$ & $72.4 \pm 3.1$ & $69.7 \pm 2.7$ & $64.5 \pm 5.6$ & $57.1 \pm 7.9$ & $74.0 \pm 1.6$ & $70.1 \pm 1.4$ & $62.6 \pm 2.4$ & $52.2 \pm 7.9$ \\
& NRGNN & $79.4 \pm 1.7$ & $77.4 \pm 2.8$ & $76.2 \pm 1.8$ & $70.4 \pm 5.6$ & $65.2 \pm 6.9$ & $79.6 \pm 1.1$ & $74.9 \pm 2.8$ & $69.8 \pm 3.2$ & $61.0 \pm 8.1$ \\
& PIGNN & $76.3 \pm 2.3$ & $74.2 \pm 4.3$ & $72.0 \pm 2.2$ & $67.8 \pm 3.1$ & $58.9 \pm 7.8$ & $76.9 \pm 2.4$ & $74.9 \pm 2.0$ & $69.3 \pm 2.6$ & $59.5 \pm 8.8$ \\
& LENS (Sum) & $77.3 \pm 0.9$ & $75.1 \pm 1.5$ & $72.9 \pm 4.5$ & $70.2 \pm 2.4$ & $61.7 \pm 5.3$ & $77.8 \pm 1.1$ & $76.5 \pm 1.4$ & $69.9 \pm 3.0$ & $62.7 \pm 3.7$ \\
& LENS (Vote) & $76.1 \pm 0.4$ & $74.3 \pm 1.9$ & $72.3 \pm 4.2$ & $68.5 \pm 2.6$ & $63.2 \pm 6.0$ & $77.2 \pm 1.7$ & $74.6 \pm 2.1$ & $68.3 \pm 2.8$ & $61.6 \pm 4.5$ \\
\midrule
\multirow{5}{*}{Amazon Ratings} 
& GCN & $39.0 \pm 0.6$ & $38.5 \pm 0.5$ & $37.7 \pm 0.9$ & $36.9 \pm 1.4$ & $36.8 \pm 0.8$ & $39.3 \pm 0.5$ & $39.1 \pm 0.5$ & $38.4 \pm 0.7$ & $36.7 \pm 1.4$ \\
& NRGNN & $37.3 \pm 0.3$ & $37.0 \pm 0.3$ & $36.9 \pm 0.4$ & $36.4 \pm 0.6$ & $36.8 \pm 0.2$ & $37.3 \pm 0.6$ & $37.9 \pm 0.7$ & $37.4 \pm 0.5$ & $36.7 \pm 1.4$ \\
& CP & $39.0 \pm 0.9$ & $38.8 \pm 0.4$ & $37.6 \pm 1.0$ & $37.1 \pm 1.0$ & $36.8 \pm 0.7$ & $39.7 \pm 0.3$ & $38.8 \pm 1.1$ & $38.6 \pm 0.7$ & $36.8 \pm 1.5$ \\
& LENS (Sum) & $42.2 \pm 0.2$ & $42.1 \pm 0.3$ & $41.9 \pm 0.3$ & $41.0 \pm 0.5$ & $40.3 \pm 0.5$ & $42.2 \pm 0.4$ & $41.6 \pm 0.5$ & $38.9 \pm 0.7$ & $36.9 \pm 0.2$ \\
& LENS (Vote) & $42.3 \pm 0.2$ & $42.2 \pm 0.4$ & $42.0 \pm 0.3$ & $41.0 \pm 0.6$ & $40.2 \pm 0.6$ & $42.3 \pm 0.4$ & $41.5 \pm 0.4$ & $38.9 \pm 0.7$ & $36.8 \pm 0.1$ \\
\midrule
\multirow{4}{*}{Roman Empire} 
& DirGNNConv & $69.1 \pm 0.5$ & $68.1 \pm 0.4$ & $67.5 \pm 0.3$ & $65.7 \pm 0.5$ & $64.8 \pm 0.4$ & $69.4 \pm 0.8$ & $68.6 \pm 0.8$ & $66.9 \pm 0.4$ & $59.9 \pm 0.6$ \\
& DeGLIF & $71.2 \pm 0.6$ & $70.6 \pm 0.2$ & $69.8 \pm 0.6$ & $68.6 \pm 0.7$ & $66.7 \pm 0.2$ & $71.1 \pm 0.5$ & $69.8 \pm 0.5$ & $67.4 \pm 0.2$ & $62.2 \pm 0.5$ \\
& LENS (Sum) & $72.6 \pm 0.8$ & $71.6 \pm 1.0$ & $70.3 \pm 0.5$ & $69.3 \pm 0.5$ & $67.3 \pm 1.1$ & $72.2 \pm 0.9$ & $70.7 \pm 0.9$ & $68.8 \pm 0.5$ & $62.2 \pm 1.3$ \\
& LENS (Vote) & $72.5 \pm 0.5$ & $71.5 \pm 0.7$ & $70.2 \pm 0.5$ & $69.2 \pm 0.3$ & $67.4 \pm 1.0$ & $72.4 \pm 0.7$ & $70.9 \pm 0.6$ & $68.8 \pm 0.7$ & $62.0 \pm 1.3$ \\
\bottomrule
\end{tabular}
}
\end{table}

LENS outperforms the evaluated methods on the Roman Empire and Amazon Ratings datasets and achieves comparable performance on PubMed. For the DBLP dataset, as noted in \cite{wang2024noisygl} NRGNN outperforms GCN by roughly 4\% even without label noise, an advantage that persists across all noise levels. By design, the accuracy of LENS is upper-bounded by GCN's performance on clean data. Despite this limitation against NRGNN, LENS still outperforms all remaining algorithms on the DBLP dataset.

\subsection{Experiments using other GNN architectures}
To ensure a fair comparison with other baselines, we initially used GCN for both Model-1 and Model-2. In this section, we evaluate the performance of LENS when both models are replaced by either GAT \cite{gatVelickovic2017GraphAN} or GraphSAGE \cite{sage}. For GAT, we used one hidden layer with 8 heads of size 8 and ReLU activation. For GraphSAGE, we used one hidden layer of size 64 with ReLU activation. The results for the Cora dataset are shown in Table \ref{tab:cora_accuracy_transposed}.

\begin{table}[!htpb]
\centering
\caption{Node classification accuracy on the Cora dataset for different backbone models under Symmetric Label Noise (SLN) and Pairwise Label Noise (PWN).}
\label{tab:cora_accuracy_transposed}
\resizebox{\linewidth}{!}{
\begin{tabular}{lc ccc ccc}
\toprule
\multirow{2}{*}{\textbf{Noise Type}} & \multirow{2}{*}{\textbf{Noise Level}} & \multicolumn{3}{c}{\textbf{GraphSAGE}} & \multicolumn{3}{c}{\textbf{GAT}} \\
\cmidrule(lr){3-5} \cmidrule(lr){6-8}
& & \textbf{Baseline GraphSAGE} & \textbf{NLD Clean (Sum)} & \textbf{NLD Clean (Vote)} & \textbf{Baseline GAT} & \textbf{NLD Clean (Sum)} & \textbf{NLD Clean (Vote)} \\
\midrule
\multirow{7}{*}{SLN} 
& 10 & $85.8 \pm 0.6$ & $86.5 \pm 0.1$ & $86.0 \pm 0.4$ & $84.7 \pm 0.5$ & $85.1 \pm 0.6$ & $85.3 \pm 0.7$ \\
& 20 & $85.0 \pm 1.0$ & $85.9 \pm 0.4$ & $85.8 \pm 0.4$ & $83.3 \pm 0.9$ & $84.3 \pm 0.6$ & $83.7 \pm 0.7$ \\
& 30 & $82.6 \pm 1.7$ & $84.4 \pm 0.7$ & $84.4 \pm 0.9$ & $83.8 \pm 0.8$ & $84.8 \pm 1.0$ & $84.3 \pm 1.0$ \\
& 40 & $79.4 \pm 1.4$ & $82.1 \pm 1.1$ & $82.9 \pm 1.6$ & $81.8 \pm 0.8$ & $82.2 \pm 1.5$ & $82.4 \pm 1.5$ \\
& 50 & $74.4 \pm 1.5$ & $79.0 \pm 1.1$ & $79.5 \pm 1.0$ & $77.0 \pm 1.4$ & $80.3 \pm 1.2$ & $80.5 \pm 1.5$ \\
& 60 & $68.7 \pm 2.4$ & $73.8 \pm 2.4$ & $74.8 \pm 1.8$ & $73.3 \pm 1.3$ & $77.0 \pm 0.8$ & $77.7 \pm 0.6$ \\
& 70 & $54.9 \pm 1.8$ & $61.2 \pm 1.0$ & $63.5 \pm 1.3$ & $56.8 \pm 2.8$ & $68.2 \pm 2.6$ & $67.3 \pm 2.6$ \\
\midrule
\multirow{4}{*}{PWN} 
& 10 & $84.8 \pm 0.7$ & $86.3 \pm 0.4$ & $86.3 \pm 0.4$ & $84.2 \pm 0.6$ & $84.8 \pm 0.9$ & $84.7 \pm 0.7$ \\
& 20 & $82.5 \pm 0.7$ & $85.3 \pm 0.8$ & $85.1 \pm 0.8$ & $81.9 \pm 0.9$ & $83.3 \pm 1.3$ & $83.7 \pm 1.5$ \\
& 30 & $74.2 \pm 1.1$ & $80.6 \pm 1.2$ & $80.0 \pm 1.3$ & $77.5 \pm 1.3$ & $80.1 \pm 0.4$ & $80.1 \pm 1.8$ \\
& 40 & $63.3 \pm 2.5$ & $69.6 \pm 3.2$ & $70.8 \pm 3.0$ & $68.4 \pm 2.6$ & $72.2 \pm 2.5$ & $72.8 \pm 2.8$ \\
\bottomrule
\end{tabular}
}
\end{table}

We observe that LENS performs well even with backbone architectures other than GCN.

}

\section{Impact of Label Noise on Interclass Distinguishability}

The NLD discussed in Section \ref{sec:nld_main} and \ref{sec:nld_res} is a node-specific quantity. Recent work \cite{Zhao2026NodeCI} extends it to a graph-level quantity, called interclass distinguishability, which measures the degree of difference between NLDs of different labels. \cite{Zhao2026NodeCI} also argues that GCNs learn well on graphs with high interclass distinguishability. Hence, it is relevant to analyze changes in interclass distinguishability in the presence of label noise. To transition from a node-level NLD to a graph-level metric, \cite{Zhao2026NodeCI} computes the neighborhood distribution pattern $c_k$ for class $k \in [K]$ by averaging the NLD vectors of all nodes belonging to that specific class $k$. These $c_k$ are then stacked as columns to construct a $K \times K$ matrix $C$, the Class NLD matrix. The entropy of the normalized singular values of this matrix $C$, which is a continuous valued proxy of the rank of this matrix, is defined as the interclass distinguishability; more details are in the Appendix \ref{app:interclassdist}. 

\begin{restatable}{theorem}{distchange}
\label{thm:nodedist}
Let $C$ be the Class NLD matrix of the balanced, clean $K$-class graph, and let $C$ be doubly stochastic. Let $\tilde{C} := Q_{\text{sln}}^\top C Q_{\text{sln}}$ denote the population-level expected Class NLD matrix under Symmetric Label Noise (SLN) with noise rate $0 < p < \frac{K-1}{K}$. Let $InterClass(\cdot)$ denote interclass distinguishability, which is defined as the entropy of the normalized singular values of a matrix. Then,  \\
$\textcolor{black}{InterClass(\tilde{C})} \leq InterClass(C).$
\end{restatable}
This theorem intuitively aligns with the node-level findings established in Theorem \ref{thm:sln_nld}. As proven earlier, SLN acts as a smoothing agent that pushes individual neighborhood label distributions toward a uniform state, increasing local entropy. When the local neighborhoods of individual nodes lose their distinct structural signatures, the aggregated class-level patterns naturally blend together.
% We argue that noise increases indistinguishability of the graph; indistinguishability is entropy of the distributions of the singular values of certain matricies formed by clean and noise corrupted NLDs, that are graph averaged. Recall that LENS considers  label of a node to be noise corrupted if its NLD crosses a certain threshold -- \textcolor{blue}{to check/refine}. 
The proof of Theorem \ref{thm:nodedist} is outlined next with details in Appendix \ref{app:interclassdist}. 

\emph{Step 1:} It is first observed that $\tilde{C}=Q_{sln}^\top CQ_{sln}=Q_{sln} CQ_{sln}$. Here is the $Q_{sln}$ transforms any vector $c_k$ to a convex combination (with $\alpha = 1-\frac{pK}{K-1}$) of $c_k$ and vector $u$ (all entries $1/K$). 

\emph{Step 2:} Matrix $C$ is shown to be sum of two matrix $C_0$ and $U$, where $U$ have all entries $1/K$, and columns of $C_0$ and $U$ are orthogonal.

\emph{Step 3:} It turns out all singular values of $\tilde{C}$, except for the highest one, are down-scaled singular values of $C$ (by factor $\alpha^2$). The highest singular values for both remain the same.

\emph{Step 4:} $\tilde{s}$ (the vector of normalized singular value of $\tilde{C}$) majorizes \cite{marshall1979inequalities} $s$ (the vector of normalized singular value of $C$). Using the Schur Concavity of the entropy function \cite{marshall1979inequalities}, we show $H(s^\prime) < H(s)$. Hence, the class distinguishability decreases, making learning hard for GCN \cite{Zhao2026NodeCI}. 

% \emph{Step 4} One can now show $H(s^\prime) < H(s)$, \emph{i.e.}, noise \textcolor{blue}{increases} entropy of distributions of \textcolor{blue}{normalized} singular values. \textcolor{blue}{ Thus indistinguishability is increased which renders (this increase of entropy renders -- to check) learning hard}

%\textcolor{blue}{we are using the converse of this; low entropy of NLD of a node means it is possibly a noisy corrupted node?}

\section{Limitations and Conclusions}

%\textbf{Limitations:}
As discussed in Section \ref{sec:high_fp}, noisy nodes residing in heterogeneous regions are difficult to identify for LENS. LENS may also not work well on datasets where nodes lack class-specific neighborhood structure, like Actor and Coronell, which are termed as graphs with ``Bad Heterophily" by \cite{mahomophily}.

In this paper, we investigate the impact of label noise on the neighborhood label distribution of a node. We observe that the presence of label noise smooths the neighborhood label distribution, leading to difficulty for Graph Convolutional Networks (GCNs) in learning from this noisy data. We formally prove this smoothing effect by demonstrating an increase in the entropy of the neighborhood label distribution. We then use this fact to propose a label-noise-robust algorithm for learning on graph data. The proposed algorithm achieves state-of-the-art performance while using significantly less compute compared to existing noise-robust approaches. We also prove that the presence of label noise reduces interclass distinguishability of a graph. 
Code is available at: \url{https://github.com/pintu-dot/nld}

\bibliographystyle{unsrtnat}
\bibliography{reference}
% For bibLaTeX users:
% \printbibliography

\appendix

\section{Proofs and Derivations}
\subsection{Proof of Theorem \ref{thm:sln_nld} and related Corollaries}
\label{proof:nldsln}
\nldsln*
\begin{proof}
    For a node $v$ and class $j$,  $\hat{c}_{j}$ in the presence of SLN is the sum of two binomial distributions. 

1. Retained: was already in class $j$, retained, and continues to stay in $j$. $c_{\mathrm{stay}}\sim Binomial(c_{j},1-p)$

2. Influx: not originally in class $j$, but influxes to class $j$ after adding noise. $c_{\mathrm{influx}}\sim Binomial(d-c_{j},\frac{p}{K-1})$.
Since noise is independently added to each node and the two sets are disjoint, $c_{\mathrm{stay}}$ and $c_{\mathrm{influx}}$ are independent. This gives  
%\textcolor{blue}{we can have space by writing these in single lines}

% \textcolor{blue}{expectation is  a linear operator; $X$ and $Y$ need not be independent}

% \textcolor{blue}{we are starting with empirical mass function of NLD and it is then changed by noise random variables}

% \textcolor{blue}{variance is not a linear operator; additive, if they are independent as then covariance is zero}

\begin{equation*}
    \begin{aligned}
\mathbb{E}[\hat{c}_{j}]&=\mathbb{E}[c_{\mathrm{stay}}]+\mathbb{E}[c_{\mathrm{influx}}]\\
&=c_j(1-p)+(d-c_j)\frac{p}{K-1}
  \end{aligned}
\end{equation*}
As \textcolor{black}{$l_j=c_j/d$}, so 

$$\textcolor{black}{\mathbb{E}[\hat{l}_{v,j}]}=l_{v,j}(1-p)+(1-l_{v,j})\frac{p}{K-1}$$
\end{proof}
% We can also look at the variance. by independence of $c_{\mathrm{stay}}$ and $c_{\mathrm{influx}}$
% \begin{equation*}
%     \begin{aligned}
%         Var(\hat{l}_j)&=Var \left (\frac{1}{d}c_{\mathrm{stay}}\right )+Var \left (\frac{1}{d}c_{\mathrm{influx}} \right)\\
%         &=\frac{1}{d^2}\left ( c_jp(1-p)+(d-c_j)\frac{p(K-1-p)}{(K-1)^2}\right )\\
%         &=\frac{1}{d}\left ( l_jp(1-p)+(d-l_j)\frac{p(K-1-p)}{(K-1)^2}\right )
%     \end{aligned}
% \end{equation*}
\emph{Alternate Proof:}
\begin{proof}
    Let the initial NLD of a node be (\emph{we drop subscript $v$ from $l_v$ for ease in notation}): $$l = \begin{bmatrix} l_1, & l_2, & \dots & l_K \end{bmatrix}$$

To find the new expected probability mass for any specific class $i$, let's call it $\mathbb{E}[\hat{l}_i]$, we multiply the distribution vector by the transition matrix (Equation \ref{matrix:sln}).

$$\mathbb{E}[\hat{l}_i] = l_i(1 - p) + \sum_{j \neq i} l_j \left( \frac{p}{K-1} \right)$$Because the sum of all probabilities is $1$, we know that $\sum_{j \neq i} l_j = 1 - l_i$. Substituting this in:$$\mathbb{E}[\hat{l}_i] = l_i(1 - p) + (1 - l_i)\left(\frac{p}{K-1}\right)$$Now, we expand and group the $l_i$ terms:$$\mathbb{E}[\hat{l}_i] = l_i - pl_i + \frac{p}{K-1} - l_i\left(\frac{p}{K-1}\right)$$$$\mathbb{E}[\hat{l}_i] = l_i \left( 1 - p - \frac{p}{K-1} \right) + \frac{p}{K-1}$$

Simplifying the term inside the parentheses:

\begin{equation*} 
\mathbb{E}[\hat{l}_i] = l_i \left( 1 - \frac{pK}{K-1} \right) + \frac{p}{K-1}
\end{equation*}
\end{proof}

\subsection{Proof of Theorem \ref{thm:pwn_nld}}
\label{proof:nldpwn}
\nldpwn*
\begin{proof}
Let $c_j$ denote the count of neighbors initially belonging to class $j$, and let $d$ be the total degree of the node. Under Pairwise Noise, instances of class $j$ retain their label with probability $1-p$, and instances from the preceding class $j-1$ flip to class $j$ with probability $p$. 

The number of neighbors remaining in class $j$ follows a binomial distribution, $c_{\mathrm{stay}} \sim \mathrm{Binomial}(c_j, 1-p)$, and the influx of neighbors from the preceding class follows $c_{\mathrm{influx}} \sim \mathrm{Binomial}(c_{j-1}, p)$, subject to the cyclic condition $c_0 \equiv c_K$. 

The expected probability mass for class $j$ is derived directly from these expected counts (\emph{we drop subscript $v$ from $l_v$ for ease in notation}):
$$\textcolor{black}{\mathbb{E}[\hat{l}_j]} = l_j(1-p) + l_{j-1}p$$

% Correspondingly, the variance is given by the sum of the variances of these independent binomial variables:
% $$Var(\hat{l}_j) = \frac{p(1-p)}{d}(c_j + c_{j-1})$$

\textbf{Entropy Inequality:} 
For the initial NLD $l=[l_1, l_2, \dots, l_K]$, we define its cyclic shift as $l_{\mathrm{shift}}=[l_K, l_1, \dots, l_{K-1}]$. 
Based on Equation \ref{eq:pwn_main}, the expected distribution vector can be expressed as a convex combination:
$$\mathbb{E}[\hat{l}] = (1-p)l + p l_{\mathrm{shift}}$$

Since $l_{\mathrm{shift}}$ is simply a permutation of $l$, their entropies are identical, meaning $H(l) = H(l_{\mathrm{shift}})$. Because the entropy function is strictly concave, the entropy of the convex combination of these distributions satisfies Jensen's inequality:
\begin{equation*}
\begin{aligned}
    H(\mathbb{E}[\hat{l})] &\geq (1-p)H(l) + pH(l_{\mathrm{shift}}) \\
               &= (1-p)H(l) + pH(l) \\
               &= H(l)
\end{aligned}
\end{equation*}
Thus, $H(\mathbb{E}[\hat{l}_v]) \geq H(l_v)$, concluding the proof.
\end{proof}

\subsection{Proof of Theorem \ref{thm:doubly_stochastic}}
\label{app:proof_arbitrary}
\nlddb*
\begin{proof}
The Birkhoff-von Neumann theorem \cite{marshall1979inequalities} states that any doubly stochastic matrix can be expressed as a convex combination of permutation matrices, say $r$ permutation matrices. So, let $Q = \sum_{j=1}^r \theta_j P_j$, where $P_j$ are permutation matrices, $\theta_j \geq 0$, and $\sum_{j=1}^r \theta_j = 1$ . 
%\textcolor{blue}{changed $k$ to $r$}

The expected NLD is $\mathbb{E}[\hat{l}] = lQ = \sum_{j=1}^r \theta_j (l P_j)$. 
Since the entropy function $H$ is strictly concave, we can apply Jensen's inequality:
\begin{equation*}
    H(\mathbb{E}[\hat{l}]) \geq \sum_{j=1}^r \theta_j H(l P_j)
\end{equation*}
Because permutation matrices simply reorder the elements of $l$ without changing their values, the entropy remains invariant, meaning $H(l P_j) = H(l)$ for all $j$. Therefore:
\begin{equation*}
    H(\mathbb{E}[\hat{l}]) \geq \sum_{j=1}^r \theta_j H(l) = H(l) \sum_{j=1}^r \theta_j = H(l)
\end{equation*}
This concludes the proof.
\end{proof}

\subsection{Counterexamples for Remark \ref{rem:counterexamples}}
\label{app:counterexamples}

The aim of the following counterexamples is to demonstrate that the entropy-increasing property of label noise is not universal; it is intrinsically tied to the doubly stochastic nature of the noise transition matrix. When these conditions are relaxed, noise can artificially collapse the distribution, leading to a decrease in neighborhood entropy.

\textbf{Counterexample 1: Non-Doubly Stochastic Matrix.}

Consider a $3$-class setting with a noise transition matrix $Q$ that is row-stochastic but not doubly stochastic (the columns do not sum to $1$):
$$Q=\begin{bmatrix}
0.5 & 0.5 & 0 \\
0.5 & 0.5 & 0\\
0.5 & 0 & 0.5
\end{bmatrix}$$
Let the initial NLD be a uniform distribution $l=[1/3, 1/3, 1/3]$. The expected NLD after noise is applied is $\mathbb{E}[\hat{l} = lQ = [1/2, 1/3, 1/6]$. The initial entropy is at its maximum for a 3-class distribution, whereas the noisy entropy is lower. Thus, $H(\mathbb{E}[\hat{l}]) < H(l)$.

\textbf{Counterexample 2: Node-Dependent Noise.}

Consider a situation where the transition probability matrix is node-specific. Even if every node's specific matrix is doubly stochastic, the overall NLD entropy might decrease.
Let the number of classes be $2$, and the noise model be dependent on the node's true class. The transition matrices are defined as: 
$$Q_{u \in \text{class } 1}=\begin{bmatrix} 1 & 0 \\ 0 & 1 \end{bmatrix} \quad \text{and} \quad Q_{u \in \text{class } 2}=\begin{bmatrix} 0 & 1 \\ 1 & 0 \end{bmatrix}$$
Suppose a node's initial NLD is $l=[0.9, 0.1]$. Because the neighbors originally in class 1 remain in class 1 (via $Q_{u \in \text{class } 1}$), and the neighbors in class 2 flip to class 1 (via $Q_{u \in \text{class } 2}$), the expected NLD concentrates entirely into the first class, resulting in $\mathbb{E}[\hat{l}] = [1, 0]$. 
The initial distribution $l$ has non-zero entropy, while the resulting distribution $\mathbb{E}[\hat{l}]$ has an entropy of zero. Therefore, $H(\mathbb{E}[\hat{l}]) < H(l)$.

\subsection{Proof for Theorem \ref{thm:hpb}}
\label{app:proof_hpb}
{\color{black}
\hpb*
\begin{proof}
    \textit{Step 1: Concentration Inequality:}
    
    Let $e_k$ denote the standard one-hot basis vector for class $k$, then 
    $$\hat{l}_v=\frac{1}{d_v}\sum_{u\in N(v)}e_{\hat{y_u}}$$

So, $\hat{l}_v$ is a realization from a multinomially distributed vector. Using Bretangolle-Huber-Carol inequality \cite{vanderVaartWellner2003ed2}, for given $\epsilon>0$

\begin{equation}
    \mathbb{P}(||\hat{l}_v-\mathbb{E}[\hat{l}_v]||_1\ge\epsilon)\le 2^K exp\left(\frac{-d_v\epsilon^2}{2}\right)
\end{equation}

Equivalently, with probability $1-\delta$

\begin{equation}
\label{eq:concentration}
    \|\hat{l}_v-\mathbb{E}[\hat{l}_v]\|_1\leq \sqrt{\frac{2~(K\ln(2)+\ln(1/\delta))}{d_v}}
\end{equation}

\textit{Step 2: Continuity of Shannon entropy}

For two distributions $\mathbb{p}$ and $\mathbb{q}$ Fannes–Audenaert gives \cite{Audenaert2006ASC}

\begin{equation}
    \label{eq:fannes:aud}
    |H(\mathbf{p})-H(\mathbf{q})\leq| T\ln(K-1)+H([T,1-T])
\end{equation}

where $T=\frac{1}{2}\sum_{i=1}^K|p_i-q_i|$.

Using Equations \ref{eq:concentration} and \ref{eq:fannes:aud}.
\begin{equation}
    |H(\hat{l}_v)-H(\mathbb{E}[\hat{l}_v])|\leq \sqrt{\frac{K\ln(2)+\ln(1/\delta)}{2~d_v}} \ln(K-1) + H([T,1-T])
\end{equation}

which is equivalent to 
\begin{equation}
    |H(\hat{l}_v)-H(\mathbb{E}[\hat{l}_v])|\leq \sqrt{\frac{K\ln(2)+\ln(1/\delta)}{2~d_v}} \ln(K-1) + \ln 2
\end{equation}
\end{proof}

We can observe that the gap decreases as the degree of the node increases. Also, if the number of classes is high, the gap is large for small degree nodes.

Next, we are interested in checking when $H(\hat{l}_v)>H(l)$ with high probability. In Theorem \ref{thm:doubly_stochastic} we showed for a doubly stochastic noise $H(\mathbb{E}[\hat{l}_v])\geq H(l_v)$, let

\begin{equation}
    \label{eq:delta_def}
    \Delta_v:=H(\mathbb{E}[\hat{l}_v])-H(l_v)\ge0
\end{equation}
Also assume

$$H(\hat{l}_v)\geq H(\mathbb{E}[\hat{l}_v])-\mathrm{error}$$

Using Equation \ref{eq:delta_def}
$$H(\hat{l}_v)\geq H(l_v)+\Delta_v-\mathrm{error}$$
We see that $H(\hat{l}_v)>H(l_v)$ if $\Delta_V\ge \mathrm{error}$. This happens with probability $1-\delta$, when 
\begin{equation}
    \Delta_v\geq C\sqrt{\frac{K+\ln(1/\delta)}{d_v}}\ln(K)
\end{equation}

Now using Equation \ref{eq:delta_def}, if $H(l_v)$ is high, then intuitively $H(\hat{l}_v)\geq H(l)$ holds only for nodes with very high degree, justifying extra filtering used by LENS. 
}

\section{Impact of label noise on Interclass distinguishability}
\label{app:interclassdist}
NLD discussed in Sections \ref{sec:nld_main} and \ref{sec:nld_res} are node-level quantities.
\cite{Zhao2026NodeCI} aggregates NLD to obtain graph-level quantities and links its learnability of GCN. One such quantity in \cite{Zhao2026NodeCI} is "Interclass Node Distinguishability", which is the degree of difference between the NLDs of nodes belonging to different class labels. To quantify this, \cite{Zhao2026NodeCI} looks at macro-patterns for each class.

The major steps in obtaining the Interclass Node distinguishability are:

%Here is how they build the metric:
\begin{itemize}
    \item 
The Class Pattern:  The average Neighborhood Label Distribution (NLD) for all nodes within a specific class $k \in [K]$ is computed. This gives a vector $c_k$ representing the neighborhood distribution of a class.  

\item The Matrix: These vectors ($c_k$) are stacked  as columns to form a $K \times K$ matrix called Class NLD.  

\item The Rank: To see how distinct these classes are from one another, the rank of this matrix is evaluated.   In an ideal scenario where every class has a unique, linearly independent neighborhood pattern, the matrix is of full rank.  

\item The Entropy Proxy: \cite{Zhao2026NodeCI} argues that, ``\textit{Because computing the strict mathematical rank is slow and binary}", hence a continous valued proxy is used. We would like to highlight that because of real-world noise, NLD would have a higher rank than its clean counterpart,  hence using a proxy also solves this. The proxy used the entropy of the matrix's normalized singular values.  Let the normalized singular values be $\sigma_i$; then, the distinguishability is defined as: 

$$InterClass = \frac{-\sum_{i=1}^K \sigma_i \log(\sigma_i)}{\log(K)}$$

A high entropy value means the singular values are distributed uniformly, indicating a full-rank, highly distinguishable matrix. A low entropy value indicates the matrix is collapsing toward rank 1, meaning the classes are indistinguishable. As we want to compare the distinguishability before and after noise, we safely ignore normalization by $log(K)$.

\end{itemize}

\distchange*

\begin{proof}

Let $C$ be our original $K \times K$ $ClassNLD$ matrix (where columns are the average NLDs for each class). Let $\sigma_1, \sigma_2, ..., \sigma_K$ be its original singular values, sorted in descending order.

\textbf{Step 1: The Matrix Transformation}

When SLN is applied to the graph with noise rate $p$, every true label has a chance to flip. As we established previously (Section
\ref{sec:nld_res}), this applies an expected transition matrix $Q_{sln}$ to the label distributions.$$\mathbb{E}[\hat{l}] = \alpha l + (1-\alpha)U[1,:]$$where $\alpha = 1 - \frac{pK}{K-1}$ (the contraction factor) and $U$ is a matrix where every entry is $\frac{1}{K}$.

The new population level expected Class NLD matrix, let us call it $\tilde{C}$, is given by ($\tilde{c}_j=\mathbb{E}[\hat{l}_v^\top|\hat{y}_v=j]$):$$\tilde{C} = Q_{sln}^\top CQ_{sln}=Q_{sln}C Q_{sln}$$
$$\implies\tilde{C} = \left(\alpha C + (1-\alpha)C U\right)Q_{sln}$$

Because every row of $C$ is a valid probability distribution (summing to 1), multiplying $C$ by the uniform matrix $U$ simply results in $U$ again.
\begin{equation}
\label{eq:noisy_C}    
\tilde{C} = \left(\alpha C + (1-\alpha)U\right)Q_{sln}=\alpha CQ_{sln}+(1-\alpha)U=\alpha^2C+(1-\alpha^2)U
\end{equation}

\textbf{Step 2: The singular value decay}

The next step is to find the singular values for the matrix $\tilde{C}$. Let $\sigma_1,\cdots, \sigma_r$ be singular values of $C$, where $\sigma_1$ is the largest singular value. 

\textbf{We claim that:} $$\sigma_1^{\prime} = \sigma_1$$$$\sigma_i^{\prime} = \alpha^2 \sigma_i \quad \text{for } i \ge 2$$

\textbf{Proof for this claim:}
{\color{black}
By definition, every column of $C$ sums to 1. This makes $C$ a column-stochastic matrix.

Step 2.1: Orthogonal Matrix Decomposition

Let $U$ be a matrix of the same size ($K \times K$) where every entry is $\frac{1}{K}$. We can write $U = \frac{1}{K}\mathbf{1}\mathbf{1}^T$, where $\mathbf{1}$ is a column vector of all ones. Because $C$ is column-stochastic, multiplying it by $U$ simply yields $U$:$$U C = U$$We can decompose the matrix $C$ into two components: its uniform mean component and its structural variance component. 
Let $C_0$ be defined as: $$C_0 = C - U$$
That is, 
$$C = U + C_0$$ %\textcolor{blue}{we need `that is'}

%Let $C_0$ be the variance component:$$C_0 = C - U$$$$C = U + C_0$$ \textcolor{blue}{we need `that is'}

Step 2.2: Orthogonality of columns of $U$ and $C_0$
 
We now show that every column of $C_0$ is orthogonal to the uniform direction. As $C$ is column-stochastic ($\mathbf{1}^T C = \mathbf{1}^T$) and $U$ is column-stochastic ($\mathbf{1}^T U = \mathbf{1}^T$), the columns of $C_0$ must sum to zero:$$\mathbf{1}^T C_0 = \mathbf{1}^T (C - U) = \mathbf{1}^T - \mathbf{1}^T = 0$$Because $\mathbf{1}^T C_0 = 0$, we have that $U^T C_0$ is a zero matrix: $$U^T C_0 = (\frac{1}{K}\mathbf{1}\mathbf{1}^T) C_0 = \frac{1}{K}\mathbf{1}(\mathbf{1}^T C_0) = \frac{1}{K}\mathbf{1}(0) = 0$$

Step 2.3: The Singular Values of the Original Matrix

The singular values of $C$ are the square roots of the eigenvalues of $C^T C$.$$C^T C = (U + C_0)^T(U + C_0)$$$$C^T C = U^T U + U^T C_0 + C_0^T U + C_0^T C_0$$Because $U^T C_0 = 0$ (and $C_0^T U = 0$), the middle cross-terms vanish. Furthermore, $U$ is an idempotent projection matrix, meaning $U^T U = U$. $ C^T C = U + C_0^T C_0$. Because the subspaces are perfectly orthogonal and $C$ is row stochastic, the eigenvalues of this sum are simply the union of their individual eigenvalues. $U$ has exactly one non-zero eigenvalue of $1$. Hence, the singular values of $C$ are exactly:$$\sigma_1 = 1$$$$\sigma_i = \sigma_{i-1}(C_0) \quad \text{for } i \ge 2$$

Step 2.4: The Singular Values of the Noisy Matrix
 
By Equation \ref{eq:noisy_C}:
$$\tilde{C} =\alpha^2(C_0+U)+U\left(1-\alpha^2\right)=U+\alpha^2C_0$$

We repeat the SVD process for $\tilde{C}$:$$\tilde{C}^T\tilde{C} = (U + \alpha^2 C_0)^T(U + \alpha^2 C_0) =  U + \alpha^4 C_0^T C_0$$
Taking the square roots of the eigenvalues, we get the exact singular values of the noisy matrix:$$\tilde{\sigma}_1 = 1$$$$\tilde{\sigma}_i \alpha^2 = \alpha^2 \sigma_i(C_0) = \alpha^2 \sigma_i \quad \text{for } i \ge 2$$

This means that the dominant singular value remains exactly 1, while all secondary singular values are scaled down by exactly the contraction factor $\alpha^2$.
}

\textbf{Step 3: Normalization and Skew}

%The paper normalizes the singular values to sum to 1 before calculating entropy.
Let $S$ be the sum of the original singular values. The original normalized values are $s_i = \frac{\sigma_i}{S}$.  Let $\tilde{S}$ be the sum of the new singular values:$$\tilde{S} = \sigma_1 + \alpha^2 \sum_{i=2}^r \sigma_i$$ 
The new normalized singular values are:$$\tilde{s}_1 = \frac{\sigma_1}{\tilde{S}}$$$$\tilde{s}_i = \frac{\alpha^2 \sigma_i}{\tilde{S}} \quad \text{for } i \ge 2$$

\textbf{Step 4: The Entropy Drop}

Because the contraction factor $\alpha^2$ is strictly less than 1, the new sum $\tilde{S}$ is smaller than the original sum $S$. 
%This causes a mathematical \textcolor{blue}{should drop the `mathematical'} phenomenon called majorization:
The dominant normalized singular value $\tilde{s}_1$ grows significantly larger than its original counterpart $s_1$.
%\textcolor{blue}{these 2 sentences are the same, said in 2 ways.} 
The smaller normalized singular values $\tilde{s}_i$ shrink relative to the dominant one. The distribution of the singular values becomes drastically more skewed (concentrated on the first index). So $\tilde{s}$ majorizes $s$ (we will prove this), and entropy being a Schur concave function \cite{marshall1979inequalities} $H(\tilde{s})<H(s)$. 
The only thing left in the proof is to prove majorization.

\textbf{Proving majorization:}

Vector $x'$ is said to majorize $x$ (assuming both are sorted in descending order and sum to 1), if the cumulative sum of the top $k$ elements of $x'$ is greater than or equal to the cumulative sum of the top $k$ elements of $x$, for all $k$ from $1$ to $K$ \citep{marshall1979inequalities}. This is denoted by $x' \succ x$. This holds for our $\tilde{s}$ and $s$ ($\tilde{s} \succ s$) using the intermediate results of Step 3. 

\emph{Claim:} For any index $k < K$:
  $$\sum_{i=1}^k s'_i \ge \sum_{i=1}^k s_i$$ 
  
  Using the above definitions of the normalized singular values, we have:
  $$\frac{\sigma_1 + \alpha^2 \sum_{i=2}^k \sigma_i}{\sigma_1 + \alpha^2 \sum_{i=2}^K \sigma_i} \ge \frac{\sigma_1 + \sum_{i=2}^k \sigma_i}{\sigma_1 + \sum_{i=2}^K \sigma_i}$$
  
  %To make the algebra cleaner, let us substitute the components with variables:
  To simplify the notation, we take 
$A : = \sigma_1$, $ B := \sum_{i=2}^k \sigma_i$ and $ C : = \sum_{i=k+1}^K \sigma_i$.

Then, the inequality becomes:$$\frac{A + \alpha^2 B}{A + \alpha^2 B + \alpha^2 D} \ge \frac{A + B}{A + B + D}$$

% We then have :$$(A + \alpha B)(A + B + C) \ge (A + B)(A + \alpha B + \alpha C)$$ 

We then have, $$A^2 + AB + AD + \alpha^2 AB + \alpha^2 B^2 + \alpha^2 BD \ge A^2 + \alpha^2 AB + \alpha^2 AD + AB + \alpha^2 B^2 + \alpha^2 BD$$

%Most terms cancel out perfectly. Subtracting the identical terms from both sides leaves only:

That is, $$AD \ge \alpha^2 AD$$ Since $A$ (the dominant singular value) and $D$ (the sum of the tail singular values) are strictly positive, and the contraction factor $\alpha$ is strictly less than 1, the value $AD$ is greater than or equal to $\alpha^2 AD$ (equal when $D=0$).  This proves that the partial sums of $\tilde{s}$ are always strictly greater than the partial sums of $s$. Therefore, $\tilde{s} \succ s$.
\end{proof}

\section{Comparison with Simple Neighborhood-Based Heuristics}
\label{app:heuristic_comparison}

We compare LENS with four heuristics similar to LENS for identifying noisy labeled nodes. These heuristics use either agreement with neighboring labels or majority labels within the two-hop neighborhood.

\paragraph{Homophily-based heuristics.}
These methods compare a node's label with the labels of its one-hop neighbors:
\begin{itemize}
    \item \textbf{Hom\_fraction:} A node is identified as noisy if its label disagrees with the labels of more than half of its neighbors.
    \item \textbf{Hom\_majority:} A node is identified as noisy if its label differs from the most frequent label among its neighbors.
\end{itemize}

\paragraph{Neighbor's neighbor-based heuristics.}
These methods additionally use label information from the two-hop neighborhood:
\begin{itemize}
    \item \textbf{Nbr\_rarity:} A node is identified as noisy if its label differs from the most frequent label among its two-hop neighbors, including the node itself.
    \item \textbf{Nbr\_majority:} A node is identified as noisy if its label disagrees with the local majority label of more than half of its one-hop neighbors.
\end{itemize}
Here, majority label refers to the most frequent label and does not necessarily require more than half of the labels to agree.

We evaluate these heuristics and both variants of LENS on Cora, a homophilic graph, and Roman Empire, a heterophilic graph, under SLN and PWN. We follow the same experimental setup as in the main paper. The node classification results are presented in Table~\ref{tab:heuristic_comparison}.

\begin{table}[!htpb]
\centering
\caption{Node classification accuracy (\%, mean $\pm$ standard deviation) of neighborhood-based heuristics and LENS on Cora, Citeseer, Amazon Photo, and Roman Empire under Symmetric Label Noise (SLN) and Pairwise Label Noise (PWN).}
\label{tab:heuristic_comparison}
\resizebox{\linewidth}{!}{
\begin{tabular}{ll ccccc cccc}
\toprule
\multirow{2}{*}{\textbf{Dataset}} & \multirow{2}{*}{\textbf{Algorithm}} & \multicolumn{5}{c}{\textbf{Symmetric Label Noise (SLN) (\%)}} & \multicolumn{4}{c}{\textbf{Pairwise Label Noise (PWN) (\%)}} \\
\cmidrule(lr){3-7} \cmidrule(lr){8-11}
& & \textbf{10} & \textbf{20} & \textbf{30} & \textbf{40} & \textbf{50} & \textbf{10} & \textbf{20} & \textbf{30} & \textbf{40} \\
\midrule
\multirow{6}{*}{Cora}
& Hom\_fraction & $86.1 \pm 0.3$ & $85.0 \pm 0.6$ & $84.9 \pm 0.3$ & $81.3 \pm 0.6$ & $79.4 \pm 0.4$ & $85.9 \pm 0.6$ & $83.9 \pm 0.8$ & $78.0 \pm 1.7$ & $65.2 \pm 2.9$ \\
& Hom\_majority & $86.1 \pm 0.4$ & $85.6 \pm 0.3$ & $84.6 \pm 0.8$ & $81.1 \pm 0.6$ & $79.5 \pm 1.1$ & $85.8 \pm 0.3$ & $84.6 \pm 0.7$ & $79.1 \pm 1.4$ & $65.7 \pm 3.1$ \\
& Nbr\_rarity & $86.7 \pm 0.4$ & $86.1 \pm 0.4$ & $84.4 \pm 0.7$ & $82.2 \pm 0.6$ & $80.3 \pm 1.0$ & $86.3 \pm 0.2$ & $83.9 \pm 1.1$ & $78.3 \pm 1.1$ & $70.3 \pm 4.2$ \\
& Nbr\_majority & $86.8 \pm 0.3$ & $86.1 \pm 0.6$ & $84.2 \pm 0.3$ & $82.4 \pm 0.7$ & $80.8 \pm 1.2$ & $86.4 \pm 0.5$ & $83.6 \pm 1.0$ & $78.5 \pm 1.5$ & $69.8 \pm 4.8$ \\
& LENS (Sum) & $86.6 \pm 0.3$ & $86.0 \pm 0.6$ & $85.3 \pm 0.3$ & $82.7 \pm 1.1$ & $81.6 \pm 1.1$ & $86.3 \pm 0.7$ & $84.2 \pm 0.9$ & $79.2 \pm 1.9$ & $70.8 \pm 2.7$ \\
& LENS (Vote) & $86.8 \pm 0.3$ & $85.9 \pm 0.6$ & $85.8 \pm 0.5$ & $82.6 \pm 1.2$ & $81.3 \pm 0.7$ & $86.3 \pm 0.3$ & $83.6 \pm 0.9$ & $78.9 \pm 1.8$ & $71.9 \pm 2.3$ \\
\midrule
\multirow{6}{*}{Citeseer}
& Hom\_fraction & $78.0 \pm 1.0$ & $76.6 \pm 0.9$ & $74.7 \pm 0.7$ & $71.8 \pm 1.4$ & $68.1 \pm 1.3$ & $78.1 \pm 0.3$ & $76.7 \pm 0.8$ & $73.6 \pm 0.4$ & $62.8 \pm 1.6$ \\
& Hom\_majority & $77.6 \pm 0.6$ & $76.6 \pm 1.0$ & $74.7 \pm 0.3$ & $71.9 \pm 1.2$ & $69.4 \pm 0.8$ & $78.3 \pm 0.3$ & $76.3 \pm 0.5$ & $73.4 \pm 0.3$ & $63.4 \pm 2.3$ \\
& Nbr\_rarity & $77.9 \pm 0.3$ & $77.2 \pm 0.5$ & $75.4 \pm 0.8$ & $72.7 \pm 0.9$ & $69.5 \pm 1.0$ & $78.2 \pm 0.6$ & $76.9 \pm 1.0$ & $72.6 \pm 0.4$ & $66.3 \pm 1.4$ \\
& Nbr\_majority & $78.2 \pm 0.3$ & $77.0 \pm 0.5$ & $75.7 \pm 0.6$ & $73.4 \pm 1.3$ & $69.8 \pm 1.0$ & $78.0 \pm 0.4$ & $76.5 \pm 0.5$ & $72.6 \pm 0.5$ & $66.4 \pm 0.8$ \\
& LENS (Sum) & $76.9 \pm 1.1$ & $75.7 \pm 0.7$ & $74.7 \pm 1.1$ & $73.7 \pm 0.6$ & $71.6 \pm 1.2$ & $76.8 \pm 1.0$ & $75.8 \pm 0.7$ & $74.3 \pm 1.1$ & $65.7 \pm 2.5$ \\
& LENS (Vote) & $77.3 \pm 1.0$ & $76.4 \pm 0.7$ & $75.2 \pm 0.8$ & $74.5 \pm 0.9$ & $72.0 \pm 1.2$ & $77.3 \pm 0.6$ & $76.0 \pm 0.9$ & $74.1 \pm 1.1$ & $65.7 \pm 2.2$ \\
\midrule
\multirow{6}{*}{Amazon Photo}
& Hom\_fraction & $91.9 \pm 0.4$ & $90.5 \pm 0.1$ & $89.2 \pm 1.6$ & $88.2 \pm 1.9$ & $83.4 \pm 4.3$ & $92.0 \pm 0.7$ & $88.8 \pm 2.8$ & $82.5 \pm 5.7$ & $69.8 \pm 7.7$ \\
& Hom\_majority & $91.8 \pm 0.6$ & $90.5 \pm 0.3$ & $89.7 \pm 1.1$ & $88.1 \pm 2.1$ & $84.2 \pm 3.9$ & $92.0 \pm 0.5$ & $88.9 \pm 2.7$ & $82.8 \pm 5.4$ & $68.8 \pm 7.1$ \\
& Nbr\_rarity & $89.1 \pm 0.7$ & $86.8 \pm 2.1$ & $88.3 \pm 1.1$ & $86.5 \pm 2.3$ & $82.9 \pm 3.1$ & $89.9 \pm 0.9$ & $89.4 \pm 1.7$ & $82.8 \pm 8.6$ & $78.7 \pm 4.6$ \\
& Nbr\_majority & $91.6 \pm 0.6$ & $89.4 \pm 2.3$ & $90.6 \pm 0.8$ & $89.2 \pm 1.2$ & $85.9 \pm 3.2$ & $91.4 \pm 0.3$ & $90.9 \pm 1.2$ & $82.7 \pm 9.0$ & $78.0 \pm 4.6$ \\
& LENS (Sum) & $91.1 \pm 0.7$ & $90.3 \pm 1.2$ & $90.4 \pm 1.1$ & $89.0 \pm 1.0$ & $87.4 \pm 1.9$ & $91.1 \pm 0.8$ & $89.2 \pm 1.1$ & $86.3 \pm 3.2$ & $77.8 \pm 4.4$ \\
& LENS (Vote) & $91.6 \pm 0.4$ & $91.1 \pm 0.7$ & $90.4 \pm 0.7$ & $89.0 \pm 1.0$ & $86.7 \pm 1.7$ & $91.0 \pm 1.3$ & $89.8 \pm 1.2$ & $86.5 \pm 2.7$ & $77.7 \pm 3.9$ \\
\midrule
\multirow{6}{*}{Roman Empire}
& Hom\_fraction & $37.7 \pm 2.7$ & $38.0 \pm 1.5$ & $36.5 \pm 1.3$ & $36.1 \pm 4.2$ & $31.6 \pm 2.8$ & $34.3 \pm 1.1$ & $31.9 \pm 2.0$ & $31.5 \pm 1.9$ & $29.8 \pm 1.0$ \\
& Hom\_majority & $69.0 \pm 0.3$ & $67.3 \pm 0.4$ & $65.3 \pm 0.8$ & $63.1 \pm 1.0$ & $59.3 \pm 0.5$ & $66.8 \pm 0.6$ & $60.2 \pm 1.0$ & $53.1 \pm 0.6$ & $45.3 \pm 0.9$ \\
& Nbr\_rarity & $72.7 \pm 0.2$ & $71.4 \pm 0.3$ & $69.3 \pm 0.8$ & $67.0 \pm 0.4$ & $65.0 \pm 0.6$ & $72.3 \pm 0.3$ & $70.1 \pm 0.4$ & $65.7 \pm 1.1$ & $54.4 \pm 0.7$ \\
& Nbr\_majority & $64.7 \pm 0.6$ & $62.4 \pm 0.5$ & $60.3 \pm 0.9$ & $57.1 \pm 0.5$ & $54.9 \pm 0.5$ & $62.2 \pm 1.0$ & $55.8 \pm 0.6$ & $46.6 \pm 0.9$ & $39.7 \pm 1.2$ \\
& LENS (Sum) & $72.6 \pm 0.8$ & $71.6 \pm 1.0$ & $70.3 \pm 0.5$ & $69.3 \pm 0.5$ & $67.3 \pm 1.1$ & $72.2 \pm 0.9$ & $70.7 \pm 0.9$ & $68.8 \pm 0.5$ & $62.2 \pm 1.3$ \\
& LENS (Vote) & $72.5 \pm 0.5$ & $71.5 \pm 0.7$ & $70.2 \pm 0.5$ & $69.2 \pm 0.3$ & $67.4 \pm 1.0$ & $72.4 \pm 0.7$ & $70.9 \pm 0.6$ & $68.8 \pm 0.7$ & $62.0 \pm 1.3$ \\
\bottomrule
\end{tabular}
}
\end{table}

On Cora, LENS and the neighbor's neighbor-based heuristics perform comparably at low noise levels, while LENS generally achieves higher accuracy as noise increases. On Roman Empire, Nbr\_rarity is the strongest heuristic and performs comparably to LENS at low noise levels. However, the advantage of LENS becomes more pronounced at higher noise levels. At 50\% SLN, LENS (Vote) achieves an accuracy of 67.4\%, compared with 65.0\% for Nbr\_rarity. At 40\% PWN, LENS (Sum) achieves 62.2\%, compared with 54.4\% for Nbr\_rarity.

These observations can be understood through the neighborhood information used by the methods. The heuristics rely on label agreement or majority labels, whereas LENS measures leave-one-out changes in the entropy of the entire neighborhood label distribution. On homophilic graphs, where neighboring labels are often concentrated in one class, majority-based decisions can closely align with those of LENS. On heterophilic graphs, useful neighborhood patterns may involve several classes and need not be captured by a single majority label. For example, two classes can share the same most frequent neighboring label while differing in the distribution of the remaining labels. By considering all class proportions through NLD entropy, LENS retains information that majority-based heuristics discard. The results support the usefulness of this distribution-based screening mechanism, particularly on heterophilic graphs at higher noise levels.

\section{Related Works}
\label{app:relatedworks}
Graph Neural Networks (GNNs) have become highly popular due to their success in modeling relational data \citep{kipf2016semi,chev,Fey/Lenssen/2019,sage}. Among these approaches, Graph Convolutional Networks (GCNs) \citep{kipf2016semi} stand out as a widely used message-passing technique. A significant body of past research has focused on the challenge of learning with label noise in graph datasets.

For example, D-GNN \citep{NT2019LearningGN} applies a backward loss correction method to handle noise. Similarly, NRGNN \citep{dai2021nrgnn} reduces the impact of noisy labels by linking unlabeled nodes to labeled ones that share similar features, which helps generate reliable pseudo-labels for better training. Coteaching+ \citep{Yu2019HowDD} builds resilience by simultaneously training two networks and updating the training data based on how confident each network is in its predictions.

RTGNN \citep{Qian2022RobustTO} tackles noise through three specific steps: bridging labeled and unlabeled nodes, using dual GCNs to find and lessen noise, and relying on deep learning memory mechanisms to enforce consistency and self-correction. To defend against adversarial label flipping, CP \citep{CPZhang2020AdversarialLA} relies on a label smoothness assumption to spot and fix noisy labels, maintaining alignment between the graph structure and node labels. RNCGLN \citep{RncglnZhu2024RobustNC} adopts a self-training approach with pseudo-labeling, where a classifier predicts labels for all nodes and replaces original labels that have low confidence scores.

PIGNN \citep{pignnDu2021NoiserobustGL} focuses on pairwise interactions between nodes, as these are more robust to noise than single node labels. It estimates interaction labels from the graph structure and uses them to regularize the classification model, keeping embeddings similar for strongly connected nodes. CGNN \citep{cgnnYuan2023LearningOG} combines graph contrastive learning for consistent representations with a sample selection technique based on homophily, allowing it to correct nodes that clash with their neighbors.

CRGNN \citep{crgnnLi2024ContrastiveLO} limits noise impacts by blending contrastive learning with a dynamic cross-entropy loss that targets nodes with stable predictions across different augmented views. TSS \citep{TSS} uses Class-conditional Betweenness Centrality to map out node positions relative to class boundaries and applies a curriculum learning strategy. This method starts training on clean nodes far from class boundaries and slowly includes harder nodes closer to the borders. DeGLIF \cite{DeGLIf} introduced a targeted denoising technique that utilizes a small set of clean data alongside the leave-one-out influence function to achieve label-noise robust node-level predictions. By estimating the change in validation loss when a training node is removed, DeGLIF identifies and relabels noisy nodes without requiring any prior knowledge of the noise model or overall noise level.

\clearpage
\newpage

\end{document}